\documentclass{article}
\usepackage{arxiv}

\usepackage[utf8]{inputenc} 
\usepackage[T1]{fontenc}    
\usepackage{hyperref}       
\usepackage{url}            
\usepackage{booktabs}       
\usepackage{amsfonts}       
\usepackage{nicefrac}       
\usepackage{microtype}      
\usepackage{lipsum}
\usepackage{fancyhdr}       
\usepackage{graphicx}       
\graphicspath{{media/}}     

\usepackage{xcolor}         
\usepackage{mathtools}

\usepackage{tikz-cd}
\usepackage{wrapfig}

\usepackage{latexsym,amsmath,amssymb,bm}
\usepackage[font=footnotesize, skip=5pt]{caption}
\usepackage{times}
\usepackage{amsthm}
\usepackage[retainorgcmds]{IEEEtrantools}
\usepackage{mdwlist}
\usepackage{enumitem}
\usepackage{listings}
\usepackage{thmtools,thm-restate}
\usepackage{thm-restate}

\usepackage{soul}

\newtheorem{theo}{Theorem}
\newtheorem{defn}{Definition}

\newtheorem{lma}{Lemma}

\newtheorem{rem}{Remark}
\newtheorem{cor}{Corollary}

\DeclareMathOperator*{\argmin}{arg\,min}
\DeclareSymbolFont{bbold}{U}{bbold}{m}{n}
\DeclareSymbolFontAlphabet{\mathbbold}{bbold}

\title{Optimal Learning Rates for Regularized Conditional Mean Embedding 
}

\author{
  Zhu Li\thanks{Equal Contribution.} \\
  Gatsby Computational Neuroscience Unit \\
  University College London \\
  \texttt{zhu.li@ucl.ac.uk} \\
  \And
  Dimitri Meunier\footnotemark[1]  \\
  Gatsby Computational Neuroscience Unit \\
  University College London \\
  \texttt{ dimitri.meunier.21@ucl.ac.uk}
   \And
  Mattes Mollenhauer \\
  Department of Mathematics and Computer Science\\
  Freie Universit\"{a}t Berlin\\
  \texttt{mattes.mollenhauer@fu-berlin.de} \\
  \And
  Arthur Gretton\\
  Gatsby Computational Neuroscience Unit \\
  University College London \\
  \texttt{arthur.gretton@gmail.com} \\
}
\begin{document}

\maketitle
\def\thefootnote{1}

\begin{abstract}
We address the consistency of a kernel ridge regression estimate of the  conditional mean embedding (CME), which is an embedding of the conditional distribution of $Y$ given $X$ into a target reproducing kernel Hilbert space $\mathcal{H}_Y$. The CME allows us to take conditional expectations of target RKHS functions, and has been employed in nonparametric causal and Bayesian inference.
We address the misspecified setting, where the target CME is
in the space of Hilbert-Schmidt operators acting from an input interpolation space between $\mathcal{H}_X$ and $L_2$, to $\mathcal{H}_Y$. This space of operators is shown to be isomorphic to a newly defined vector-valued interpolation space. Using this isomorphism, we derive a novel and adaptive statistical learning rate for the empirical CME estimator under the misspecified setting. Our analysis reveals that our rates match the optimal $O(\log n / n)$ rates without assuming $\mathcal{H}_Y$ to be finite dimensional. We further establish a lower bound on the learning rate, which shows that the obtained upper bound is optimal.
\end{abstract}

\section{Introduction}
Approximation of the conditional expectation operator is a central issue in the statistical learning community, and many approaches have been proposed \cite{williams2015data,klus2015numerical,klus2018data,korda2018convergence}. Given random variables $X$ and $Y$, the conditional expectation operator for a function $f$ is defined
\begin{IEEEeqnarray}{rCL}
\left[Pf\right](x):= \mathbb{E}\left[f(Y)|X = x\right]. \nonumber
\end{IEEEeqnarray}

Conventional parametric models to approximate $P$ often involve density estimation and expensive numerical analysis. Hence, recent studies attempt to explore a new framework to approximate $P$ via kernel methods. Specifically, given kernels $k_X$ and $k_Y$ with corresponding reproducing kernel Hilbert space $\mathcal{H}_X$ and $\mathcal{H}_Y$ for $X$ and $Y$ respectively, we may define the conditional mean embedding (CME) as $F_*(x) := \mathbb{E}[k_Y(\cdot,Y)|X= x]$, and we may  employ the reproducing property to obtain $[Pf](x) = \langle f, F_*(x) \rangle_{\mathcal{H}_Y}$ for any $f \in \mathcal{H}_Y$.
The advantage of the CME framework is that it
allows the  straightforward evaluation of conditional expectations of any function in $\mathcal{H}_Y$.
The CME framework has been applied successfully to many learning problems such as probabilistic inference \cite{song2013kernel}, reinforcement learning \cite{nishiyama2012hilbert,grunewalder2012modelling} and causal inference \cite{mitrovic2018causal,singh2019kernel}.

Despite these successful applications,
there have been two main challenges in establishing a rigorous theory of CMEs.
The first challenge, remarkably, has been in establishing a principled and sufficiently general definition of the conditional mean embedding itself.
The CME was originally introduced as an operator mapping from $\mathcal{H}_X$ to $\mathcal{H}_Y$
\cite{fukumizu2004dimensionality,song2009hilbert}. This definition has the benefit of elegance, and of a straightforward expression
in terms of feature covariances and cross-covariances. 
A disadvantage is that the definition requires the conditional mean $\mathbb{E}[g(Y)|X=\cdot]\in \mathcal{H}_X,\:\forall g\in \mathcal{H}_Y$.
This strong assumption may be violated in practice (see \cite{klebanov2020rigorous, klebanov2021linear} and \cite[Section 3.1]{fukumizu2013kernel} for illustrations and alternative
requirements),
and significantly restricts the class of distributions on which we can define a CME.

An alternative approach, due to  \cite{grunewalder2012conditional}, is to express the conditional mean embedding as the solution of a least-squares regression
problem in a vector-valued RKHS \cite{carmeli2006vector,carmeli2010vector}.
In subsequent work, a rigorous measure-theoretic definition of the conditional mean embedding as the
$\mathcal{H}_Y$-valued square
integrable function $F_*$ is established in \cite{park2020measure,klebanov2020rigorous}, which is the definition we will use in the present work.
Both \cite{grunewalder2012conditional,park2020measure} connect this
CME definition to the original operator-mapping definition by means of a surrogate loss, which upper bounds the regression loss. 
A direct connection remained elusive until the work of \cite{mollenhauer2020nonparametric,klebanov2021linear},
which show that under denseness assumptions,
the CME can be arbitrarily well approximated by a Hilbert-Schmidt operator from $\mathcal{H}_X$ to $\mathcal{H}_Y$,
thus connecting the operator-theoretic and measure-theoretic definitions.

The second challenge has been in obtaining consistency results and the optimal learning rates for empirical estimates of the CME.
An early consistency analysis of the sample estimator, due to \cite{song2010nonparametric}, requires very strong smoothness assumptions.
A more refined analysis, due to
\cite{grunewalder2012conditional}, attains the minimax optimal learning rate $O(\log n /n)$  for the sample estimator, but only in the case where $\mathcal{H}_Y$ is finite dimensional. For the infinite dimensional RKHS, \cite{singh2019kernel} and \cite{park2020measure} establish consistency in the well-specified case, with learning rates of $O(n^{-1/6})$ and $O(n^{-1/4})$. Nevertheless, the obtained rates are far from optimal and consistency under misspecified setting was not established. Recently, \cite{talwai2021sobolev} obtains a sharper rate under the misspecified case using the interpolation RKHS.  
The results of \cite{talwai2021sobolev}
impose assumptions, however, which strongly limit their applicability (refer to Remark \ref{rem:inter_RKHS} for a rigorous discussion):
\begin{enumerate}\itemsep0pt
    \item They require an explicit relation between the smoothness of the target CME and the size of the RKHS. In particular, when the kernel has slow eigenvalue decay (as in  the case of Mat{\'e}rn kernels, for example), the setting is very close to the well-specified scenario.
    \item They rely on the explicit construction of an interpolation RKHS.
    Unlike in \cite{fischer2020sobolev}, where a similar approach is based
    only on equivalence classes of functions (i.e., Sobolev-like spaces), 
    this concept requires the embedding of the RKHS into the corresponding 
    $L^2$-space (or equivalently the integral operator) 
    to be injective---which is generally not the case
    (see \cite{steinwart2012mercer} for details). 
    Counterexamples can easily be constructed when one considers
    degenerate pushforward measures on the RKHS  in one or more coordinate directions
    (for example point masses). 
    By contrast, the authors of \cite{fischer2020sobolev} do
    not explicitly require the injectivity in the real-valued learning scenario. Moreover, in case where the chosen kernel has slow eigenvalue decay, the constructed interpolation RKHS is not well-defined.
\end{enumerate}
Finally, to our knowledge, there is presently no result establishing a matching lower bound for the CME learning rate
in the case where $\mathcal{H}_Y$ may be infinite dimensional. Hence, whether the obtained upper rate is optimal remains unknown.

In the present work, we address the challenges mentioned above.  
Building on \cite{park2020measure,mollenhauer2020nonparametric} and the interpolation space theory results of 
\cite{steinwart2012mercer, fischer2020sobolev},
we introduce an \textbf{interpolation space consisting of vector-valued
functions} via a natural tensor product construction.
This concept is compatible with the recent measure-theoretic definition
of the CME due to \cite{park2020measure}
and allows to prove convergence in the misspecified setting 
without the  limitations of prior work.
Based on this novel vector-valued interpolation space, we establish {\bf consistency and convergence rates of the CME sample estimator in the misspecified setting}. In particular, under certain benign conditions, we obtain the optimal $O(1 /n)$ learning rate up to a logarithmic factor. This matches with the current optimal analysis from \cite{grunewalder2012conditional} without the restrictive assumption of finite dimensional $\mathcal{H}_Y$. Thanks to our operator-theoretic definition of the CME, and unlike \cite{talwai2021sobolev}, we do not require an a-priori relation between the rate of kernel eigenvalue decay and the smoothness of the conditional mean operator (i.e., our results apply generally in the misspecified setting). Finally, in Theorem~\ref{theo:lower_bound}, we provide a novel {\bf lower bound on the CME learning rate}, which demonstrates that the obtained upper rate is optimal in the setting of a smooth CME operator.

\section{Background}\label{sec:bg}
Throughout the paper, we consider two random variables $X$, $Y$ defined respectively on the second countable locally compact Hausdorff spaces $E_X$ and $E_Y$ endowed with their respective Borel $\sigma$-field $\mathcal{F}_{E_X}$ and $\mathcal{F}_{E_Y}$. We let $(\Omega, \mathcal{F},\mathbb{P})$ be the underlying probability space with expectation operator $\mathbb{E}$. Let $\pi$ and $\nu$ be the pushforward of $\mathbb{P}$ under $X$ and $Y$ respectively, i.e., $X \sim \pi$ and $Y \sim \nu$. We use the Markov kernel $p: E_X \times \mathcal{F}_{E_Y} \rightarrow \mathbb{R}_+$ to define the conditional distribution: \[\mathbb{P}[Y \in A|X = x] = \int_{A} p(x,dy),\] for all $x \in E_X$ and events $A \in \mathcal{F}_{E_Y}$. We denote the space of real-valued Lebesgue square integrable functions on $(E_X,\mathcal{F}_{E_X})$ with respect to $\pi$ as $L_2(E_X,\mathcal{F}_{E_X},\pi)$ abbreviated $L_2(\pi)$ and similarly for $\nu$ we use $L_2(E_Y,\mathcal{F}_{E_Y},\nu)$ abbreviated $L_2(\nu)$. Let $B$ be a separable Banach space with norm $\|\cdot\|_B$ and $H$ a separable real Hilbert space with inner product $\langle \cdot, \cdot \rangle_{H}$. We write $\mathcal{L}(B,B')$ as the Banach space of bounded linear operators from $B$ to another Banach space $B'$, equipped with the operator norm $\|\cdot\|_{B\rightarrow B'}$. When $B=B'$, we simply write $\mathcal{L}(B)$ instead. We also let $L_p(E_X,\mathcal{F}_{E_X},\pi; B)$, abbreviated $L_p(\pi; B)$, the space of strongly $\mathcal{F}_{E_X}-\mathcal{F}_B$ measurable and Bochner $p$-integrable functions $f: E_X \rightarrow B$ for $1 \leq p \leq \infty$. Finally, we denote the $p$-Schatten class $S_p(H,H')$ to be the space of all compact operators $C$ from $H$ to another Hilbert space $H'$ such that $\|C\|_{S_p(H,H')} := \left\|\left(\sigma_i(C)\right)_{i\in J}\right\|_{\ell_p}$ is finite. Here $\|\left(\sigma_i(C)\right)_{i\in J}\|_{\ell_p}$ is the $\ell_p$ sequence space norm of the sequence of the strictly positive singular values of $C$ indexed by the countable set $J$. For $p = 2$, $S_2(H,H')$ is the Hilbert space of Hilbert-Schmidt operators from $H$ to $H'$. 

{\bf Tensor Product of Hilbert Spaces (\cite{aubin2000applied}, Section 12):} 
Denote $H\otimes H'$  the tensor product of Hilbert spaces $H$, $H'$. The Hilbert space $H\otimes H'$ is the completion of the algebraic tensor product with respect to the norm induced by the inner product $\langle x_1\otimes x_1', x_2\otimes x_2'\rangle_{H\otimes H'} = \langle x_1,x_2 \rangle_H \langle x_1', x_2'\rangle_{H'}$ for $x_1,x_2 \in H$ and $x_1', x_2' \in H'$ defined on the elementary tensors
of $H\otimes H'$. This definition extends to
$\operatorname{span}\{x\otimes x'| x\in H, x'\in H'\}$ and finally to
its completion. The space $H\otimes H'$ is separable whenever both $H$ and $H'$ are separable. The element $x\otimes x' \in H \otimes H'$ is treated as the linear rank-one operator $x\otimes x': H' \rightarrow H$ defined by $y' \rightarrow \langle y',x' \rangle_{H'}x $ for $y' \in H'$. 
Based on this identification, the tensor product space $H\otimes H'$ is isometrically isomorphic to the space of Hilbert-Schmidt operators from $H'$ to $H$, i.e., $H\otimes H' \simeq S_2(H',H)$. We will hereafter not make the distinction between those two spaces and see them as identical. If $\{e_i\}_{i \in I}$ and $\{e'_j\}_{j \in J}$ are orthonormal basis in $H$ and $H'$, $\{e_i\otimes e'_j\}_{i\in I, j\in J}$ is an orthonormal basis in $H \otimes H'$. 

\begin{rem}[\cite{aubin2000applied}, Theorem 12.6.1]\label{rem:tensor_product}
Consider the Bochner space $L_2(\pi;H)$ where $H$ is a separable Hilbert space. One can show that $L_2(\pi;H)$ is isometrically identified with the tensor product space $H \otimes L_2(\pi)$.
\end{rem}

{\bf Reproducing Kernel Hilbert Spaces, Covariance Operators:} We let $k_{X}: E_X \times E_X \rightarrow \mathbb{R}$ be a symmetric and positive definite kernel function and $\mathcal{H}_{X}$ be a vector space of $E_X \rightarrow \mathbb{R}$ functions, endowed with a Hilbert space structure via an inner product $\langle\cdot, \cdot\rangle_{\mathcal{H}_{X}}$. $k_{X}$ is a reproducing kernel of $\mathcal{H}_{X}$ if and only if: 1. $\forall x \in E_X, k_{X}(\cdot,x) \in \mathcal{H}_{X} ; 2 . \forall x \in E_X$ and $\forall f \in \mathcal{H}_{X}, f(x)=\left\langle f, k_{X}(x,\cdot)\right\rangle_{\mathcal{H}_{X}}$.   A space $\mathcal{H}_{X}$ which possesses a reproducing kernel is called a reproducing kernel Hilbert space (RKHS)\cite{berlinet2011reproducing}. We denote the canonical feature map of $\mathcal{H}_{X}$ as $\phi_{X}(x) = k_{X}(\cdot,x)$. Similarly for $E_Y$, we consider a RKHS $\mathcal{H}_{Y}$ with symmetric and positive definite kernel $k_{Y}: E_Y \times E_Y \rightarrow \mathbb{R}$ and canonical feature map denoted as $\phi_{Y}$.

We require some technical assumptions on the previously defined RKHSs and kernels:
\begin{itemize}[itemsep=0em]
    \item[$1$.] $\mathcal{H}_{X}$ and $\mathcal{H}_{Y}$ are separable, this is satisfied if $E_X$ and $E_Y$ are Polish spaces and 
    $k_{X}, k_{Y}$ are continuous \cite{steinwart2008support};
    
    \item[$2$.]  $k_{X}(\cdot,x)$ and $k_{Y}(\cdot,y)$ are measurable for $\pi$-almost all $x \in E_X$ and $\nu$-almost all $y \in E_Y$;
    
    \item[$3$.] $k_X(x,x) \leq \kappa_X^2$ for $\pi$-almost all $x \in E_X$ and $k_Y(y,y) \leq \kappa_Y^2$ for $\nu$-almost all $y \in E_Y$.
\end{itemize}

Note that the above assumptions are not restrictive in practice, as well-known kernels such as the Gaussian, Laplacian and Mat{\'e}rn kernels satisfy all of the above assumptions on $\mathbb{R}^d$. We now introduce some facts about the interplay between $\mathcal{H}_{X}$ and $L_{2}(\pi),$ which has been extensively studied by \cite{smale2004shannon,smale2005shannon},\cite{de2006discretization} and \cite{steinwart2012mercer}. We first define the (not necessarily injective) embedding $I_{\pi}: \mathcal{H}_{X} \rightarrow L_{2}(\pi)$, mapping a function $f \in \mathcal{H}_{X}$ to its $\pi$-equivalence class $[f]$. The embedding is a well-defined compact operator as long as its Hilbert-Schmidt norm is finite. In fact,
this requirement is satisfied since its Hilbert-Schmidt norm can be computed as 
$$
\left\|I_{\pi}\right\|_{S_{2}\left(\mathcal{H}_{X}, L_{2}(\pi)\right)}=\|k_{X}\|_{L_{2}(\pi)}:=\left(\int_{E_X} k_{X}(x, x) \mathrm{d} \pi(x)\right)^{1 / 2}<\infty.
$$
The adjoint operator $S_{\pi}:=I_{\pi}^{*}: L_{2}(\pi) \rightarrow \mathcal{H}_{X}$ is an integral operator with respect to the kernel $k_{X}$, i.e. for $f \in L_{2}(\pi)$ and $x \in E_X$ we have
$$
\left(S_{\pi} f\right)(x)=\int_{E_X} k_{X}\left(x, x^{\prime}\right) f\left(x^{\prime}\right) \mathrm{d} \pi\left(x^{\prime}\right)
$$
Next, we define the self-adjoint and positive semi-definite integral operators
$$
L_{X}:=I_{\pi} S_{\pi}: L_{2}(\pi) \rightarrow L_{2}(\pi) \quad \text { and } \quad C_{XX}:=S_{\pi} I_{\pi}: \mathcal{H}_{X} \rightarrow \mathcal{H}_{X}
$$
These operators are trace class and their trace norms satisfy
$$
\left\|L_{X}\right\|_{S_{1}\left(L_{2}(\pi)\right)}=\left\| C_{XX}\right\|_{S_{1}(\mathcal{H}_{X})}=\left\|I_{\pi}\right\|_{S_2\left(\mathcal{H}_{X},L_{2}(\pi)\right)}^{2}=\left\|S_{\pi}\right\|_{S_2\left(L_{2}(\pi),\mathcal{H}_{X}\right)}^{2} .
$$

\paragraph{Vector-valued RKHS} We also give a brief overview of the vector-valued reproducing kernel Hilbert space (vRKHS). We refer the reader to \cite{carmeli2006vector} and \cite{carmeli2010vector} for more detail.

\begin{defn}
Let $H$ be a real Hilbert space and $K: E_X \times E_X \rightarrow \mathcal{L}(H)$ be an operator valued positive-semidefinite (psd) kernel such that $K(x,x') = K(x',x)^*$ for all $x,x' \in E_X$, and for all $x_1,\dots,x_n\in E_X$ and $h_i,h_j \in H$,
\[\sum_{i,j =1}^n \langle h_i,K(x_i,x_j)h_j \rangle_{H} \geq 0.\]
\end{defn}

Fix $K$, $x \in E_X$, and $h \in H$, $\left[K_{x} h\right](\cdot):=K(\cdot, x) h$
defines a function from $E_X$ to $H$. We now consider
$$
\mathcal{G}_{\text {pre }}:=\operatorname{span}\left\{K_{x} h \mid x \in E_X, h \in H\right\}
$$
with inner product on $\mathcal{G}_{\text {pre }}$ by linearly extending the expression
\begin{IEEEeqnarray}{rCl}
\left\langle K_{x} h, K_{x^{\prime}} h^{\prime}\right\rangle_{\mathcal{G}}:=\left\langle h, K\left(x, x^{\prime}\right) h^{\prime}\right\rangle_{H} . \label{eqn:vrkhs_inp}
\end{IEEEeqnarray}

Let $\mathcal{G}$ be the completion of $\mathcal{G}_{\text{pre}}$ with respect to this inner product. We call $\mathcal{G}$ the vRKHS induced by the kernel $K$. The space $\mathcal{G}$ is a Hilbert space consisting of functions from $E_X$ to $H$ with the reproducing property
\begin{IEEEeqnarray}{rCl}
\langle F(x), h\rangle_{H}=\left\langle F, K_{x} h\right\rangle_{\mathcal{G}}, \label{eqn:vrkhs_repro}
\end{IEEEeqnarray}
for all $F \in \mathcal{G}, h \in H$ and $x \in E_X$. For all $F \in \mathcal{G}$ we obtain
$$
\|F(x)\|_{H} \leq\|K(x, x)\|^{1 / 2}\|F\|_{\mathcal{G}}, \quad x \in E_X.
$$
Since the inner product given by Eq.~(\ref{eqn:vrkhs_inp}) implies that $K_{x}$ is a bounded operator for all $x \in E_X$. For all $F \in \mathcal{G}$ and $x \in E_X$, Eq.~(\ref{eqn:vrkhs_repro}) can be written as $F(x)=K_{x}^{*} F$. The linear operators $K_{x}: H \rightarrow \mathcal{G}$ and $K_{x}^{*}: \mathcal{G} \rightarrow H$ are bounded with
$$
\left\|K_{x}\right\|=\left\|K_{x}^{*}\right\|=\|K(x, x)\|^{1 / 2}
$$
and we have $K_{x}^{*} K_{x^{\prime}}=K\left(x, x^{\prime}\right), x, x^{\prime} \in E_X$. In the following, we will denote $\mathcal{G}$ as the vRKHS induced by the kernel $K: E_X \times E_X \rightarrow \mathcal{L}(\mathcal{H}_{Y})$ with \[K(x,x') := k_{X}(x,x')\text{Id}_{\mathcal{H}_{Y}}, x,x' \in E_X.\]

An important property of $\mathcal{G}$ is that elements in $\mathcal{G}$ are isometric to Hilbert-Schmidt operators between $\mathcal{H}_{X}$ and $\mathcal{H}_{Y}$.

\begin{theo}[Theorem $4.4$ in \cite{mollenhauer2020nonparametric}]\label{theo:isometric}
Let $\mathcal{H}_{X}$ and $\mathcal{H}_{Y}$ be real-valued RKHS with kernel $k_{X}$ and $k_{Y}$ respectively. For $f_{Y} \in \mathcal{H}_{Y}$ and $g_{X} \in \mathcal{H}_{X}$, define the map $\bar{\Psi}$ on the elementary tensors as 
\[\left[\bar{\Psi}\left(f_{Y} \otimes g_{X}\right)\right](x):= g_{X}(x)f_{Y} = \left(f_{Y} \otimes g_{X}\right)\phi_{X}(x).\] 
We then have that $\bar{\Psi}$ defines an isometric isomorphism between $S_2(\mathcal{H}_{X}, \mathcal{H}_{Y})$ and $\mathcal{G}$ through linearity and completion.
\end{theo}

More details regarding Theorem~\ref{theo:isometric} can be found in \cite[Theorem 4.4]{mollenhauer2020nonparametric}. The isometric isomorphism $\bar{\Psi}$ induces the operator reproducing property stated below.
\begin{cor}\label{theo:operep}
For every function $F\in \mathcal{G}$ there exists an operator $C := \bar{\Psi}^{-1}(F) \in S_2(\mathcal{H}_{X}, \mathcal{H}_{Y})$ such that \[F(x) = C\phi_{X}(x) \in \mathcal{H}_{Y},\] for all $x \in E_X$ with $\|C\|_{S_2(\mathcal{H}_{X}, \mathcal{H}_{Y})} = \|F\|_{\mathcal{G}}$ and vice versa. Conversely, for any pair $F \in \mathcal{G}$ and $C \in S_2(\mathcal{H}_{X}, \mathcal{H}_{Y})$, we have $C = \bar{\Psi}^{-1}(F)$ as long as $F(x) = C \phi_{X}(x)$.
\end{cor}
The proof of Corollary~\ref{theo:operep} is a simple extension of Lemma $15$ in \cite{ciliberto2016consistent} and Corollary $4.5$ in \cite{mollenhauer2020nonparametric}. Corollary ~\ref{theo:operep} shows that the vRKHS $\mathcal{G}$ is generated via the space of Hilbert-Schmidt operators $S_2(\mathcal{H}_{X}, \mathcal{H}_{Y})$ \[\mathcal{G}= \left\{F: E_X \rightarrow \mathcal{H}_{Y}|F = C\phi_{X}(\cdot), C \in S_2(\mathcal{H}_{X}, \mathcal{H}_{Y})\right\}.\]

{\bf Conditional Mean Embedding:}
A particular advantage of kernel methods is its convenience of operating probability distributions, see \cite{muandet2017kernel,sejdinovic2013equivalence} for examples. This is through the so called kernel mean embedding \cite{berlinet2011reproducing,smola2007hilbert,kanagawa2018gaussian}.
Assuming the integrability condition
$\int_{E_X} \sqrt{k_{X}(x, x)} d \pi(x)<\infty$ (which is satisfied when the kernel is almost surely bounded), we define the kernel mean embedding $\mu_{X}(\cdot)=\int_{E_X} k_{X}(\cdot,x) d \pi(x)$. It is easy to show that for each $f \in \mathcal{H}_{X}, \int_{E_X} f(x) d \pi(x)=\left\langle f, \mu_{X}\right\rangle_{\mathcal{H}_{X}}$. Replacing $\pi$ with the conditional distribution, we obtain the kernel conditional mean embedding as defined in \cite{park2020measure,klebanov2020rigorous}.

\begin{defn}
The $\mathcal{H}_{Y}$-valued conditional mean embedding (CME) for the Markov kernel $p(x,dy)$ is defined as 
\begin{IEEEeqnarray}{rCl}
F_*(x):=\int_{E_Y}\phi_{Y}(y)p(x,dy) = \mathbb{E}\left[\phi_{Y}(Y)|X = x \right] \in L_2(E_X,\mathcal{F}_{E_X},\pi; \mathcal{H}_{Y}) \label{eqn:CME}
\end{IEEEeqnarray}
\end{defn}

By the reproducing property, we have $\mathbb{E}[f_{Y}(Y)|X=x] = \langle f_{Y}, F_*(x) \rangle_{\mathcal{H}_{Y}}, \forall f_{Y} \in \mathcal{H}_{Y}$ and $x \in E_X$. The approximation of $F_*$ is a key concept in kernel methods. By \cite{mollenhauer2020nonparametric}, suppose we impose Assumptions $1$-$3$ together with two additional assumptions: i) $\mathcal{H}_{X} \subseteq C_0(E_X)$ where $C_0(E_X)$ is the space of continuous functions vanishing at infinity\footnote{This is satisfied if $k_{X}$ is bounded and $k_{X}(\cdot,x) \in C_0(E_X)$ for $\pi$-almost all $x \in E_X$.}; and ii) $\mathcal{H}_X$ is dense in $L_2(\pi)$, then we have that $\mathcal{G}$ is dense in $L_2(E_X, \mathcal{F}_{E_X}, \pi; \mathcal{H}_Y)$. As a result, for any $\delta > 0$, there is an $F \in \mathcal{G}$ such that $\|F- F_*\|_{L_2} < \delta$. Hence, in the literature, we often assume the so-called \emph{well-specified case} to obtain a closed-form solution,
\begin{IEEEeqnarray}{rCl}
F_* \in \mathcal{G}. \label{eqn:cme_assum1}
\end{IEEEeqnarray}
It is shown in \cite[Theorem $5.3$]{klebanov2020rigorous} and
\cite[Corollary $5.6$ and Remark $5.8$]{mollenhauer2020nonparametric}  that $F_*$ admits a closed form expression under Eq.~(\ref{eqn:cme_assum1}) via
\[F_*(x) = (C_{XX}^{\dagger}C_{XY})^*\phi_{X}(x),\] where 
$C_{YX} = \mathbb{E}[\phi_{Y}(Y)\otimes \phi_{X}(X)]$ and $C^{\dagger}$ denotes the pseudoinverse of $C$.

\begin{rem}
We point out that in the original derivations, the CME is written as $F_*(x) = C_{YX}C_{XX}^{\dagger}\phi_{X}(x)$ \cite{song2009hilbert,fukumizu2004dimensionality,fukumizu2013kernel}. However, $C_{XX}^{\dagger}$ is not globally defined
if $\mathcal{H}_X$ is infinite-dimensional. 
Hence the expression $C_{YX}C_{XX}^{\dagger}\phi_{X}(x)$ is problematic, as we expect $F_{*}$ to be defined for all $x \in E_X$ based on 
the Markov kernel $p$. In the well-specified scenario, \cite{klebanov2020rigorous} corrected this issue by defining the CME as $(C_{XX}^{\dagger}C_{XY})^*\phi_{X}(x)$. It is shown that in this case,
$(C_{XX}^{\dagger}C_{XY})^*$ is bounded 
(actually Hilbert-Schmidt, see also \cite{klebanov2021linear}), and hence globally defined. The connection of this  corrected operator-theoretic 
perspective to the well-specified regression scenario was established
in \cite{mollenhauer2020nonparametric}.
\end{rem}

Once we have the closed-form solution, a natural question to ask is how to estimate the CME. Indeed, this has been extensively studied in  \cite{grunewalder2012conditional,park2020measure,talwai2021sobolev}. Given a data set $D = \{(x_i, y_i)\}_{i=1}^n$ independently and identically sampled from the joint distribution of $X$ and $Y$, a regularized estimate of $F_*$ is the solution of the following optimization problem:
\begin{IEEEeqnarray}{rCl}
\hat{C}_{Y|X,\lambda}:= \argmin_{C \in S_2(\mathcal{H}_{X}, \mathcal{H}_{Y})} \frac{1}{n}\sum_{i = 1}^n \left\|\phi_Y(y_i) -C \phi_X(x_i)\right\|^2_{\mathcal{H}_Y} + \lambda \|C\|_{S_2(\mathcal{H}_{X}, \mathcal{H}_{Y})}^2, \label{eqn:emp_cme}
\end{IEEEeqnarray}
$\hat{F}_{\lambda}(\cdot) := \bar{\Psi}\left(\hat{C}_{Y|X,\lambda}\right)(\cdot)  = \hat{C}_{Y|X,\lambda}\phi_X(\cdot)$, where $\lambda$ is the regularization parameter. Implicit in the construction, however, is the  assumption 
$F_* \in \mathcal{G}$ that the solution is well-specified.
We provide a few remarks
regarding this assumption:

\begin{rem}
i) In the literature, the prevalent definition of well-specifiedness is through 
\begin{IEEEeqnarray}{rCl}
\mathbb{E}\left[f_{Y}(Y)|X = \cdot\right] \in \mathcal{H}_{X}, \forall f_{Y} \in \mathcal{H}_{Y}, \label{eqn:cme_assum2}
\end{IEEEeqnarray}
see e.g. \cite{fukumizu2004dimensionality,song2009hilbert,klebanov2020rigorous} for details. However, this definition is not equivalent to that in Eq.~(\ref{eqn:cme_assum1}). 
Specifically, assuming $F_* \in \mathcal{G}$ implies that Eq.~(\ref{eqn:cme_assum2}) holds. Nonetheless, the reverse is not true. In particular, there exist concrete examples satisfying Eq.~(\ref{eqn:cme_assum2}), but the corresponding
operator representative of the CME is not Hilbert--Schmidt
(see Section~\ref{sec:well_spec} in Appendix for details). To avoid confusion, we refer to Eq.~(\ref{eqn:cme_assum1}) as the well-specified case hereafter.

ii) The conventional assumption Eq.~(\ref{eqn:cme_assum2}) can actually be refined via the inclusion map $I_{\pi}$. In particular, since $I_{\pi} $ is an inclusion map from $\mathcal{H}_X$ to $L_2(\pi)$ and $\mathbb{E}\left[f_{Y}(Y)|X = \cdot\right] \in \mathcal{H}_{X}$, we can apply $I_{\pi}$ to  $\mathbb{E}\left[f_{Y}(Y)|X = \cdot\right]$. In addition, we are only interested in the case where $I_{\pi} \left( \mathbb{E}\left[f_{Y}(Y)|X = \cdot\right]\right) \neq 0$. As a result, the refined definition should be 
\begin{IEEEeqnarray}{rCl}
\mathbb{E}\left[f_{Y}(Y)|X = \cdot\right] \in \left(\operatorname{ker} I_{\pi}\right)^{\perp}, \forall f_{Y} \in \mathcal{H}_{Y}. \label{eqn:cme_assum3}
\end{IEEEeqnarray}
\end{rem}

We now characterize the Hilbert spaces used to define the CME in the  misspecified setting.

{\bf Real-valued Interpolation Space:}
We  review the results of 
\cite{steinwart2012mercer,fischer2020sobolev}
that set out the eigendecompositions of $L_{X}$ and $C_{XX}$, and apply these in constructing the interpolation spaces  used for the misspecified setting. By the spectral theorem for self-adjoint compact operators, there exists an at most countable index set $I$, a non-increasing sequence $(\mu_i)_{i\in I} > 0$, and a family $(e_i)_{i \in I} \in \mathcal{H}_{X}$, such that $\left([e_i]\right)_{i \in I}$ is an orthonormal basis (ONB) of $\overline{\text{ran}~I_{\pi}} \subseteq L_2(\pi)$ and $(\mu_i^{1/2}e_i)_{i\in I}$ is an ONB of $\left(\operatorname{ker} I_{\pi}\right)^{\perp} \subseteq \mathcal{H}_{X}$, and we have \[L_{X} = \sum_{i\in I} \mu_i \langle\cdot, [e_i] \rangle_{L_2(\pi)}[e_i], \qquad C_{XX} = \sum_{i \in I} \mu_i \langle\cdot, \mu_i^{\frac{1}{2}}e_i \rangle_{\mathcal{H}_{X}} \mu_i^{\frac{1}{2}}e_i .\] 

For $\alpha \geq 0$, we define the $\alpha$-interpolation space \cite{steinwart2012mercer} by
\[[\mathcal{H}]_{X}^{\alpha}:=\left\{\sum_{i \in I} a_{i} \mu_{i}^{\alpha / 2}\left[e_{i}\right]:\left(a_{i}\right)_{i \in I} \in \ell_{2}(I)\right\} \subseteq L_{2}(\pi),\]
equipped with the $\alpha$-power norm
\[\left\|\sum_{i \in I} a_{i} \mu_{i}^{\alpha / 2}\left[e_{i}\right]\right\|_{[\mathcal{H}]_{X}^{\alpha}}:=\left\|\left(a_{i}\right)_{i \in I}\right\|_{\ell_{2}(I)}=\left(\sum_{i \in I} a_{i}^{2}\right)^{1 / 2}.\] More broadly, we sometimes need to deal with function of the form $f + c$ where $f \in [\mathcal{H}]_X^{\alpha}$ and $c \in \mathbb{R}$. For this, we follow the classical definition of the direct sum of two Hilbert spaces \cite{conway2019course} and define the $\alpha$-power norm for $f+c$ as
\begin{IEEEeqnarray}{rCl}
\|f+c\|_{[\mathcal{H}]_X^{\alpha}} = \|c\|_{\mathbb{R}} + \|f\|_{[\mathcal{H}]_X^{\alpha}}. \label{eqn:cons_f_norm}
\end{IEEEeqnarray}
For $\left(a_{i}\right)_{i\in I} \in \ell_{2}(I)$, the $\alpha$-interpolation space becomes a Hilbert space with inner product defined as \[\left\langle \sum_{i \in I}a_i(\mu_i^{\alpha/2}[e_i]), \sum_{i \in I}b_i(\mu_i^{\alpha/2}[e_i]) \right\rangle_{[\mathcal{H}]_{X}^{\alpha}} = \sum_{i \in I} a_i b_i.\] Moreover, $\left(\mu_{i}^{\alpha / 2}\left[e_{i}\right]\right)_{i \in I}$ forms an ONB of $[\mathcal{H}]_{X}^{\alpha}$ and consequently $[\mathcal{H}]_{X}^{\alpha}$ is a separable Hilbert space. In the following, we use the abbreviation $\|\cdot\|_{\alpha}:=\|\cdot\|_{[\mathcal{H}]_{X}^{\alpha}}$. For $\alpha=0$ we have $[\mathcal{H}]_{X}^{0}=\overline{\operatorname{ran} I_{\pi}} \subseteq L_{2}(\pi)$ with $\|\cdot\|_{0}=\|\cdot\|_{L_{2}(\pi)}$. Moreover, for $\alpha=1$ we have $[\mathcal{H}]_{X}^{1}=\operatorname{ran} I_{\pi}$ and $[\mathcal{H}]_{X}^{1}$ is isometrically isomorphic to the closed subspace $\left(\operatorname{ker} I_{\pi}\right)^{\perp}$ of $\mathcal{H}_{X}$ via $I_{\pi}$, i.e. $\left\|[f]\right\|_{1}=\|f\|_{\mathcal{H}_{X}}$ for $f \in\left(\operatorname{ker} I_{\pi}\right)^{\perp}$. For $0<\beta<\alpha$, we have
\begin{IEEEeqnarray}{rCl} \label{eqn:inter_inclusion}
[\mathcal{H}]_{X}^{\alpha} \hookrightarrow [\mathcal{H}]_{X}^{\beta}  \hookrightarrow [\mathcal{H}]_{X}^{0} \subseteq L_{2}(\pi). \label{eqn:inter_less_L_2}
\end{IEEEeqnarray}


\begin{rem} \label{rem:universality}
    Under assumptions $1$-$3$ and $E_X$ being a second-countable locally compact Hausdorff space, if $k_X(\cdot, x)$ is continuous and vanishing at infinity, then $[\mathcal{H}]_{X}^{0} = L_{2}(\pi)$ if and only if $\mathcal{H}$ is dense in the space of continuous functions vanishing at infinity equipped with the uniform norm \cite{carmeli2010vector}. Such RKHS are called $c_0-$universal. As a special case of Proposition 5.6 in \cite{carmeli2010vector}, one can show that on $\mathbb{R}^d$, Gaussian, Laplacian, inverse multiquadrics and Matérn kernels are $c_0$-universal.
\end{rem}

\begin{figure}
\centering
\begin{tikzcd} 
S_2(L_2(\pi), \mathcal{H}_Y) \arrow[r, "\Psi"']  &  L_2(\pi;\mathcal{H}_Y) \\
S_2(\mathcal{H}_X, \mathcal{H}_Y) \arrow[r, "\bar{\Psi}"'] \arrow[u, bend left, dotted, "\mathcal{I}_{\pi}", blue, xshift=-4ex] & \mathcal{G}
\end{tikzcd}
\caption{$\Psi$ and $\bar{\Psi}$ are the bijective linear operators that define the respective isomorphisms between each pair of spaces. The precise form of the isomorphisms is given in the appendix. $\mathcal{I}_{\pi}$ denotes the canonical embedding between the two Hilbert-Schmidt spaces.}
\label{fig:spaces}
\end{figure}

\section{Approximation of CME with Vector-valued Interpolation Space}
In this section, we deal with the misspecified setting where $F_* \notin \mathcal{G}$. To do this, we first define the \textit{vector-valued interpolation space} via the tensor product space. We now recall from Remark~\ref{rem:tensor_product} that $L_2(\pi;\mathcal{H}_Y)$ is isomorphic to $S_2\left(L_2(\pi), \mathcal{H}_Y \right)$ and we denote by $\Psi$ the isomorphism between the two spaces. Similarly, we have $\mathcal{G} \simeq S_2(\mathcal{H}_X, \mathcal{H}_Y)$ and we denote by $\bar\Psi$ the isomorphism between both spaces in accordance with Theorem~\ref{theo:isometric}. This is summarized in Figure \ref{fig:spaces}. The second chain of spaces is not isometric to the first but can be naturally embedded into the first as follows. Recall that we denote by $I_{\pi}: \mathcal{H}_X \rightarrow L_2(\pi)$ the embedding that maps each function to its equivalent class, $I_{\pi}(f) = [f]$. We therefore naturally define the embedding $\mathcal{I}_{\pi}: S_2(\mathcal{H}_X, \mathcal{H}_Y) \rightarrow S_2(L_2(\pi), \mathcal{H}_Y)$ through $\mathcal{I}_{\pi}(g \otimes f) = g \otimes I_{\pi}(f) = g \otimes [f]$ for all $f \in \mathcal{H}_X$, $g \in \mathcal{H}_Y$, and obtain the extension to the whole space by linearity and continuity. Therefore, for $F \in \mathcal{G}$ we define $[F] :=  \Psi \circ \mathcal{I}_{\pi} \circ \bar{\Psi}^{-1}(F)$. In the rest of the paper, every embedding will be denoted using the notation $[~\cdot~]$. Strict notation would require us to write $[~\cdot~]_{\pi}$ due to dependence on the measure $\pi$, but we omit the subscript for ease of notation.

\begin{defn}\label{def:inter_ope_norm}
Suppose that we are given real-valued kernels $k_X$ and $k_Y$ with associated RKHS $\mathcal{H}_X$ and $\mathcal{H}_Y$ and let $[\mathcal{H}]_X^{\alpha}$ be the real-valued interpolation space associated to $\mathcal{H}_X$ with some $\alpha \geq 0$. Since $[\mathcal{H}]_X^{\alpha} \subseteq L_2(\pi)$, it is natural to define the vector-valued interpolation space $[\mathcal{G}]^{\alpha}$ as
\[[\mathcal{G}]^{\alpha} 
:= \Psi\left(S_2([\mathcal{H}]_X^{\alpha},\mathcal{H}_Y) \right) = \{F \mid F = \Psi(C), ~C \in S_2([\mathcal{H}]_X^{\alpha},\mathcal{H}_Y)\}.\]
$[\mathcal{G}]^{\alpha}$ is a Hilbert space equipped with the norm $$\|F\|_{\alpha} := \|C\|_{S_2([\mathcal{H}]_X^{\alpha}, \mathcal{H}_Y)} \qquad (F \in [\mathcal{G}]^{\alpha}),$$ where $C = \Psi^{-1}(F)$. For $\alpha = 0$, we retrieve, $$\|F\|_{0} = \|C\|_{S_2(L_2(\pi), \mathcal{H}_Y)}.$$
\end{defn}

\begin{rem}\label{rem:vector_inter} The vector-valued interpolation space $[\mathcal{G}]^{\alpha}$ allows us to study the CME in the misspecified case. To see this, we note that by Eq.~(\ref{eqn:CME}), we have $F_* \in L_2(E_X,\mathcal{F}_{E_X},\pi;\mathcal{H}_Y)$. In light of Eq.~(\ref{eqn:inter_inclusion}), for $0< \beta <\alpha$ we have \[[\mathcal{G}]^{\alpha} \hookrightarrow [\mathcal{G}]^{\beta}  \hookrightarrow [\mathcal{G}]^{0}\subseteq L_2(E_X,\mathcal{F}_{E_X},\pi;\mathcal{H}_Y).\] 


While the well-specified case corresponds to $F_* \in \mathcal{G}$, the misspecified case corresponds to $F_* \in [\mathcal{G}]^{\beta}$ for some $0 \leq \beta <1$, and relaxes the well-specified assumption in Eq.~(\ref{eqn:cme_assum1}). One can see from Remark~\ref{rem:universality} that under assumptions $1$ to $3$ and $E_X$ being a second-countable locally compact Hausdorff space, $[\mathcal{G}]^{0} = L_2(\pi;\mathcal{Y})$ if and only if $k_X$ is $c_0$-universal.
\end{rem}

\section{Learning Rate for CME}
In this section, we derive the learning rate for the difference between $[\hat{F}_{\lambda}]$ and $F_*$ in the interpolation norm. We first state additional assumptions that are needed in our derivations. As our assumptions match those of \cite{fischer2020sobolev}, we include the corresponding labels from \cite{fischer2020sobolev} for ease of reference.
\begin{itemize}
    \item[$5$.] Recall that $(\mu_i)_{i\in I}$ are the eigenvalues of $C_{XX}$. For some constants $c_2 >0$ and $p \in (0,1]$ and for all $i \in I$,
    \begin{equation}\label{asst:evd}
        \mu_i \leq c_2i^{-1/p}\tag{EVD}
    \end{equation}
    
    \item[$6$.] For $\alpha \in (p, 1]$, the inclusion map $I^{\alpha, \infty}_{\pi}: [\mathcal{H}]_{X}^{\alpha} \hookrightarrow L_{\infty}(\pi)$ is continuous, and there is a constant $A > 0$ such that
    \begin{equation}\label{asst:emb}
        \|I^{\alpha, \infty}_{\pi}\|_{[\mathcal{H}]_{X}^{\alpha} \rightarrow L_{\infty}(\pi)} \leq A \tag{EMB}
    \end{equation}
    
    \item[$7$.] There exists $0 < \beta \leq 2$ such that
    \begin{equation}\label{asst:src}
         F_* \in [\mathcal{G}]^{\beta} \tag{SRC}
    \end{equation}
    We let $C_{Y|X} := \Psi^{-1}(F_*) \in S_2([\mathcal{H}]_X^{\beta}, \mathcal{H}_Y)$ and we call $C_{Y|X}$ the conditional mean embedding operator.
\end{itemize}

\eqref{asst:evd} is a standard assumption on the eigenvalue decay of the integral operator (see more details in \cite{caponnetto2007optimal,fischer2020sobolev,talwai2021sobolev}). \eqref{asst:emb} is referred as the embedding property in \cite{fischer2020sobolev} and it can be shown that it implies $\sum_{i \in I} \mu_i^{\alpha} e_i^2(x) \leq A^2$ for $\pi$-almost all $x \in E_X$ (\cite{fischer2020sobolev} Theorem 9). Since we assume $k_X$ to be bounded, the embedding property always hold true when $\alpha = 1$. Furthermore, \eqref{asst:emb} implies a polynomial eigenvalue decay of order $1/\alpha$, which is why we take $\alpha \geq p$. \eqref{asst:src} is justified by Remark~\ref{rem:vector_inter} and is often referred as the source condition in literature (\cite{caponnetto2007optimal,fischer2020sobolev,lin2018optimal,lin2020optimal}). It imposes the smoothness assumption on the target CME operator $F_*$. In particular, when $\beta \geq 1$, the source condition implies that $F_*$ has a representative from $\mathcal{G}$, indicating the well specified scenario. However, once we let $\beta < 1$, we are in the misspecified learning setting, which is the main interest in this manuscript. Finally, in computing the learning rate for real-valued regression, one often needs the so-called (MOM) condition on the Markov kernel $p(x,dy)$ (see \cite{caponnetto2007optimal,fischer2020sobolev,talwai2021sobolev} for more details). The generalization to our setting would amount to assume that there exists constants $\sigma, R > 0$ such that
\[\mathbb{E} \left[ \|\phi_{Y}(Y)- F_{*}(x)\|_{\mathcal{H}_Y}^q \mid X=x\right] \leq \frac{1}{2}q!\sigma^2R^{q-2},\]
for $\pi$-almost surely all $x \in E_X$ and all $q \geq 2$. The reason for requiring (MOM) in the scalar regression setting is that we do not usually have $|Y| < \infty$ almost surely.  In our setting, however,  Assumption $3$ implies $\|\phi_{Y}(y)- F_{*}(x)\|_{\mathcal{H}_Y} \leq 2 \kappa_Y$ for $\pi$-almost all $x \in E_X$ and $\nu$-almost all $y \in E_Y$. Therefore, (MOM) is automatically satisfied with $\sigma = R = 2\kappa_Y$.

\begin{rem}\label{rem:inter_RKHS}
We remark that in \cite{talwai2021sobolev}, a variant of SRC condition is employed. In particular, instead of assuming $F_* \in [\mathcal{G}]^{\beta}$, they impose the assumption that $F_* \in \mathcal{G}^{\beta} \simeq S_2(\mathcal{H}_X^{\beta}, \mathcal{H}_Y)$, where $\mathcal{H}_{X}^{\beta}$ is an RKHS with corresponding kernel defined as $k_{X}^{\beta}(\cdot,x) = \sum_i \mu_i^{\beta} e_i(\cdot)e_i(x).$ Comparing to $[\mathcal{G}]^{\beta} \simeq S_2([\mathcal{H}]_X^{\beta}, \mathcal{H}_Y)$, there are two shortcomings that arise when working with $\mathcal{G}^{\beta}$.

First, $\mathcal{H}_X^{\beta}$ denotes the interpolating RKHS consisting of continuous functions only, while $[\mathcal{H}]_X^{\alpha}$ is the interpolating Hilbert space, where elements are defined through an equivalence classes. Hence, by working with $\mathcal{G}^{\beta}$, the implicit assumption is that $F_*$ is a continuous function. On the other hand, assuming $F_* \in [\mathcal{G}]^{\beta}$ avoids the continuity requirement. In particular, we have $\mathcal{H}_X^{\beta} \subseteq [\mathcal{H}]_X^{\beta}$ for any $\beta > 0$. Therefore, our \eqref{asst:src} condition applies to a more general setting.

Second, and more importantly, the construction of $\mathcal{H}_X^{\beta}$ relies on the condition that $\sum_i \mu_i^{\beta} e_i^2(x) < \infty,$  as pointed out in \cite{steinwart2012mercer}. Failing this, the kernel $k_X^{\beta}$ associated with the interpolating RKHS $\mathcal{H}_X^{\beta}$ is unbounded, indicating $\mathcal{H}_X^{\beta}$ is not well-defined. Under \eqref{asst:evd}, this effectively amounts to requiring that $\beta \geq p$. When kernel has slow eigenvalue decay (as for the Mat{\'e}rn kernel), $p$ can be close to $1$. Results obtained using $\mathcal{G}^{\beta}$ while requiring $\beta > p$ are very close to the well-specified case. By contrast, $[\mathcal{H}]_X^{\beta}$ is always well-defined as a subspace of $L_2(\pi)$ for any $\beta > 0$, even if $\sum_i \mu_i^{\beta} e_i^2(x) < \infty$ does not hold. 
\end{rem}

\begin{rem}
It is important to assume $\beta > 0$ in \eqref{asst:src}, as our results do not apply when  $\beta = 0$. The $\beta = 0$ setting arises for instance when  $Y \perp X$, or when both $Y=X$ and $\mathcal{H}_Y=\mathcal{H}_X$. In the former case, it is easy to see that the CME is the constant function (w.r.t $X$) $F_*(x) = \mu_Y = \int_{E_Y}\phi_Y(y) d\nu(y).$ In the latter case, $F_*(x) = \phi_X(x)$, and the CME is the identity operator. These functions are not covered by $[\mathcal{G}]^{\beta}$ for any $\beta > 0$ (see Appendix~\ref{sec:well_spec} for details).
\end{rem}
We now provide an upper bound on the learning rate.
\begin{theo}\label{theo:cme_rate}
Let Assumptions $1$-$3$, \eqref{asst:evd}, \eqref{asst:emb} and \eqref{asst:src} with $0 <\beta \leq 2$ hold, and let $0\leq \gamma \leq 1$ with $\gamma < \beta$,

\begin{enumerate}
    \item In the case $\beta + p \leq \alpha$ and $\lambda_n = \Theta \left(\left(n/\log^r(n)\right)^{-\frac{1}{\alpha}}\right)$, for some $r > 1$, there is a constant $K > 0$ independent of $n \geq 1$ and $\tau \geq 1$ such that \[\left\|[\hat{F}_{\lambda}] - F_*\right\|^2_{\gamma} \leq \tau^2 K\left(\frac{n}{\log ^{r} n}\right)^{-\frac{\beta-\gamma}{\alpha}}\] is satisfied for sufficiently large $n \geq 1$ with $P^n$-probability not less than $1-4e^{-\tau}$. 
    \item In the case $\beta + p > \alpha$ and $\lambda_n = \Theta \left(n^{-\frac{1}{\beta + p}}\right)$, there is a constant $K > 0$ independent of $n \geq 1$ and $\tau \geq 1$ such that \[\left\|[\hat{F}_{\lambda}] - F_*\right\|^2_{\gamma} \leq \tau^2 K n^{-\frac{\beta-\gamma}{\beta + p}}\] is satisfied for sufficiently large $n \geq 1$ with $P^n$-probability not less than $1-4e^{-\tau}$.
\end{enumerate}
\end{theo}

Theorem~\ref{theo:cme_rate} provides the finite sample $\gamma$-norm learning rate for the empirical CME estimator defined in Eq.~(\ref{eqn:emp_cme}). It states that the learning rate for $[\hat{F}_{\lambda}]$ is governed by the interplay between $p$, $\alpha$, and $\beta$. Intuitively, $p$ describes the decay rate of the eigenvalues $(\mu_i)_{i \in I}$,  $\alpha$ determines the boundedness of the interpolation kernel (and has maximum value of $1$ according to our assumption), $\beta$ characterizes the smoothness of the target CME operator.

To simplify the discussion, we may focus on the $L_2(E_X,\mathcal{F}_{E_X},\pi; \mathcal{H}_Y)$ learning rate, corresponding to $\gamma = 0$. The exponent $\beta / \max\{\alpha, \beta+p\}$ explicitly provides the learning rate for the CME operator. For example, if we have $\alpha \leq \beta$, we obtain a learning rate of $\beta/ (\beta + p)$. 
In particular, for a Gaussian kernel on a bounded convex set $E_X$ with $\pi$ uniform on $E_X$, $p$ and $\alpha$ are arbitrarily close to $0$ (see Remark~\ref{rem:gauss} below), and our learning rate can achieve 
$O(\log(n)/ n)$ rate simply by taking $\lambda_n = \Theta\left(\left(\log(n)/n\right)^{1/\beta}\right)$.
\def\thefootnote{2}
\begin{rem}\label{rem:gauss} Let $E_X \subset \mathbb{R}^d$ be a bounded set with
Lipschitz boundary\footnote{As an example any bounded convex set has Lipschitz boundary.}, let $\pi$ be the uniform distribution on $E_X$ and let $k_X$ be a Gaussian kernel. Then by Corollary 4.13 in \cite{kanagawa2018gaussian}, Assumption~\ref{asst:emb} is satisfied with any $\alpha \in (0,1]$, implying that Assumption~\ref{asst:evd} is also satisfied with any $p \in (0,1]$.
\end{rem}

\section{Lower Bound}
Our final theorem provides a lower bound for the convergence rate, which allows us to confirm the optimality of our learning rate. In deriving the lower bound, we need an extra assumption
\begin{itemize}
    \item[$8$.] For some constants $c_1,c_2 >0$ and $p \in (0,1]$ and for all $i\in I$,
    \begin{equation}\label{asst:evd+}
        c_1 i^{-1/p} \leq \mu_i \leq c_2i^{-1/p}\tag{EVD+}
    \end{equation}
\end{itemize}

\begin{theo}\label{theo:lower_bound} Let $k_X$ be a kernel on $E_X$ such that Assumptions $1$-$3$ hold and $\pi$ be a probability distribution on $E_X$ such that \eqref{asst:evd+} and \eqref{asst:emb} hold $0< p\leq \alpha \leq 1$. Then for all $0 < \beta \leq 2$, $0 \leq \gamma \leq 1$ with $\gamma < \beta$ there exist constants $K_0, K, s > 0$ such that for all learning methods $D\rightarrow \hat{F}_{D}$ ($D:=\{(x_i,y_i)\}_{i=1}^n$), all $\tau > 0$, and all sufficiently large $n \geq 1$ there is a distribution $P$ defined on $E_X \times E_Y$ used to sample $D$, with marginal distribution $\pi$ on $E_X,$ such that \eqref{asst:src} with respect to $\beta$ is satisfied, and with $P^n$-probability not less than $1-K_0\tau^{1/s}$, \[\|[\hat{F}_D] - F_*\|^2_{\gamma} \geq \tau^2 K n^{-\frac{\max\{\alpha,\beta\}-\gamma}{\max\{\alpha, \beta\}+p}} .\]
\end{theo}

Theorem~\ref{theo:lower_bound} states that under the assumtions of Theorem~\ref{theo:cme_rate} and \eqref{asst:evd+}, no learning method can achieve a learning rate faster than $n^{-\frac{\max\{\alpha,\beta\}}{\max\{\alpha, \beta\}+p}}$ in $L_2$ norm. To our knowledge, this is the first analysis that demonstrates the lower rate for CME learning. In the context of regularized regression, \cite{caponnetto2007optimal,steinwart2009optimal,blanchard2018optimal} provide a similar lower bound on the learning rate. However, a key difference in our analysis is that the output of the regression learning now lives in an infinite dimensional RKHS $\mathcal{H}_Y$, rather than in $\mathbb{R}$. Our analysis reveals that in the case where $\alpha \leq \beta$, the obtained upper rate in Theorem~\ref{theo:cme_rate} is optimal, i.e., $O(n^{-\frac{\beta}{\beta + p}})$. In particular, when $k_X$ is an exponentiated quadratic kernel on a compact set $E_X \subset \mathbb{R}^d$ with Lipschitz boundary, \eqref{asst:emb} is satisfied with any $\alpha \in (0,1)$ \cite[see Corollary $4.13$]{kanagawa2018gaussian}. As a result, the optimal rate is attained as long as $\beta > 0$. We point out that finding the optimal rate for $\beta < \alpha$ remains a challenge, and is an open problem when the output is $\mathbb{R}$.

\begin{rem}
  Theorem~\ref{theo:lower_bound} states that the upper bound and the lower bound match when  $\beta > \alpha$. In particular, for exponential kernels such as Gaussian  and Laplacian kernels with subgaussian distributions, the eigenvalues for the covariance operator have geometric decay rate. In these cases, $\alpha$ is arbitrarily close to $0$. Hence, as long as $\beta > 0$, we will have $\alpha < \beta$. In other words, for commonly used kernels, CME learning can always obtain the optimal fast rate $n^{-\frac{\beta}{\beta+p}}$.
\end{rem}

\section{Conclusion}
In this paper, we provide a rigorous theoretical foundation for approximating the CME operator, and study the statistical learning rate. Utilizing recently developed interpolation space techniques, we first define the vector-valued interpolation space $[\mathcal{G}]^{\alpha}$. This  allows to define the target CME operator in the larger interpolation space $[\mathcal{G}]^{\alpha}$, in contrast to the well-specified setting where $F_* \in \mathcal{G}$. By doing so, we are able to study the convergence rate of the empirical CME operator in the misspecified scenario. We then provide a $\gamma$-norm learning rate for the CME without any assumption on the interplay between $\beta$ and $p$, with matching lower bound. Our analysis shows that under appropriate conditions, we can obtain a fast $O(\log n / n)$ convergence rate, which matches the rate obtained in the existing literature for finite dimensional $\mathcal{H}_Y$. In more challenging settings, we  still obtain the minimax optimal rate $O(n^{-1/2})$.

Looking beyond the present work, our current interpolation space setting indicates that the convergence rate can be arbitrarily slow if $\beta \rightarrow 0$. This prevents learning the constant function, which plays a crucial role in completing the theory of the CME, as pointed out by \cite{klebanov2020rigorous}. Addressing this challenge is an important direction of future  research. 

\paragraph{Acknowledgement:} The authors wish to thank Peter Orbanz and Bharath Sriperumbudur for fruitful discussions and proofreading. Zhu Li, Dimitri Meunier, and Arthur Gretton were supported by the Gatsby Charitable Foundation.
Mattes Mollenhauer was supported by the Deutsche Forschungsgemeinschaft (DFG) through grant EXC 2046 ``MATH+''; Project Number 390685689, Project IN-8 ``Infinite-Dimensional Supervised Least Squares Learning as a Noncompact Regularized Inverse Problem''.

\bibliography{ref}
\bibliographystyle{abbrv}
\newpage
\appendix

\section{Proof of Theorem \ref{theo:cme_rate}} \label{sec:learning_rate}
\paragraph{Sketch of Proof.} 
Recall that $\hat{F}_{\lambda} \in \mathcal{G}$ is defined as $\hat{F}_{\lambda} := \bar{\Psi}\left(\hat{C}_{Y|X,\lambda}\right)$ where $\hat{C}_{Y|X,\lambda}$ is solution of Eq.~(\ref{eqn:emp_cme}). We introduce the theoretical estimator that solves the regression problem in population,
\begin{IEEEeqnarray}{rCl}
C_{Y|X,\lambda}:= \argmin_{C \in S_2(\mathcal{H}_{X}, \mathcal{H}_{Y})} \mathbb{E}_{XY} \left\|\phi_Y(Y) -C \phi_X(X)\right\|^2_{\mathcal{H}_Y} + \lambda \|C\|_{S_2(\mathcal{H}_{X}, \mathcal{H}_{Y})}^2, \qquad F_{\lambda} := \bar{\Psi}\left(C_{Y|X,\lambda}\right)  \label{eqn:pop_cme}
\end{IEEEeqnarray}
It can be readily shown (see for example \cite{mollenhauer2020nonparametric})  that 
\begin{equation*}
\begin{aligned}
    C_{Y|X,\lambda} &= C_{YX}\left(C_{XX} + \lambda Id_{\mathcal{H}_X}\right)^{-1} \\
    \hat{C}_{Y|X,\lambda} &= \hat{C}_{YX}\left(\hat{C}_{XX} + \lambda Id_{\mathcal{H}_X} \right)^{-1},
\end{aligned}
\end{equation*}
where $Id_{\mathcal{H}_X}$ is the identity operator and
\begin{equation}
\begin{aligned}
    C_{XX} &= \mathbb{E}[\phi_{X}(X)\otimes \phi_{X}(X)] \qquad~~~~~ C_{YX} = \mathbb{E}[\phi_{Y}(Y)\otimes \phi_{X}(X)]& \\
    \hat{C}_{XX} &= \frac{1}{n}\sum_{i=1}^n\phi_{X}(x_i) \otimes \phi_{X}(x_i) \qquad\hat{C}_{YX} = \frac{1}{n}\sum_{i=1}^n\phi_{Y}(y_i) \otimes \phi_{X}(x_i).&
\end{aligned}
\end{equation}
Finally, recall that the CME $F_{*}$ is in $L_2(E_X, \mathcal{F}_{E_X}, \pi; \mathcal{H}_Y)$ and the CME \emph{operator} is defined as $C_{Y \mid X} := \Psi^{-1}\left(F_{*}\right)$. From the definition of the vector-valued interpolation norm we introduce the following decomposition,
\begin{IEEEeqnarray}{rCl}
\left\|[\hat{F}_{\lambda}] - F_* \right\|_{\gamma} &\leq& \left\|\left[\hat{F}_{\lambda} - F_{\lambda}\right]\right\|_{\gamma} + \left\|[F_{\lambda}] -F_*\right\|_{\gamma} \\ &=& \left\|\left[\hat{C}_{Y|X,\lambda} - C_{Y|X,\lambda}\right]\right\|_{S_2\left([\mathcal{H}]_X^{\gamma},\mathcal{H}_{Y}\right)} + \left\|[C_{Y|X,\lambda}] -C_{Y \mid X} \right\|_{S_2\left([\mathcal{H}]_X^{\gamma},\mathcal{H}_{Y}\right)} \label{eqn:risk_decom}
\end{IEEEeqnarray}

We can see that the error for the first term is mainly due to the sample approximation. We therefore refer to the first term as the \textit{Variance}.  We refer to the second term as the \textit{Bias}. Our proof of convergence of the bias adapts the proof in \cite[Theorem $6$]{singh2019kernel} and \cite{fischer2020sobolev}, and  utilizes the fact that $C_{Y|X}$ is Hilbert-Schmidt to obtain a sharp rate.

\subsection{Bounding the Bias}
In this section, we establish the bound on the bias. The key insight is that thanks to \cite[Theorem 12.6.1]{aubin2000applied}, the conditional mean embedding can be expressed as a Hilbert-Schmidt operator in the misspecified case. We then exploit the proof techniques from the bias consistency result of \cite[Theorem 6]{singh2019kernel} and \cite{fischer2020sobolev}.

\begin{restatable}{lma}{newbias}\label{lma:new_bias}
If $F_{*} \in [\mathcal{G}]^{\beta}$ is satisfied for some $0 \leq \beta \leq 2$, then the following bound is satisfied, for all $\lambda > 0$ and $ 0 \leq \gamma \leq \beta:$
\begin{equation} 
    \left\|[F_{\lambda}] - F_{*}\right\|_{\gamma}^{2} \leq\left\|F_{*}\right\|_{\beta}^{2} \lambda^{\beta-\gamma} \label{eq:lemma3_1}
\end{equation}   
\end{restatable}

\begin{proof}
We first recall that since $F_{*} \in [\mathcal{G}]^{\beta}$, $F_{*} = \Psi\left(C_{Y \mid X} \right)$ with $C_{Y \mid X} \in S_2([\mathcal{H}]_X^{\beta}, \mathcal{H}_Y)$, furthermore $F_{\lambda} = \bar{\Psi}\left(C_{Y|X,\lambda}\right)$ with $C_{Y|X,\lambda} \in S_2(\mathcal{H}_X, \mathcal{H}_Y)$. Hence, $\left\|[F_{\lambda}] - F_{*}\right\|_{\gamma} = \left\|[C_{Y \mid X, \lambda}] - C_{Y \mid X}\right\|_{S_2\left([\mathcal{H}]_X^{\gamma},\mathcal{H}_{Y}\right)}$ and $\left\|F_{*}\right\|_{\beta} = \left\|C_{Y \mid X}\right\|_{S_2\left([\mathcal{H}]_X^{\beta},\mathcal{H}_{Y}\right)}$. We first decompose $[C_{Y \mid X, \lambda}]- C_{Y \mid X}$, followed by computing the upper bound of the bias. Since $C_{Y \mid X} \in S_2([\mathcal{H}]_X^{\beta}, \mathcal{H}_Y) \subseteq  S_2(\overline{\text{ran}~I_{\pi}}, \mathcal{H}_Y)$, $C_{Y \mid X}$ admits the decomposition
\begin{equation*}
    C_{Y\mid X}=\sum_{i \in I} \sum_{j \in J} \check{a}_{ij} d_j \otimes [e_{i}].
\end{equation*}
where $(d_j)_{j \in J}$ is any basis of $\mathcal{H}_Y$ and $\sum_{i \in I} \sum_{j \in J} \check{a}_{ij}^2 < +\infty$ with $\check{a}_{ij}=\left\langle C_{Y \mid X},d_j \otimes [e_{i}] \right\rangle_{S_2(L_2(\pi), \mathcal{H}_Y)}  = \left\langle C_{Y \mid X} [e_i],d_j\right\rangle_{\mathcal{H}_Y}$ for all $i \in I, j \in J$ (see e.g.\ \cite{gretton2013introduction}, Lecture on ``testing statistical dependence''). On the other hand, $C_{Y \mid X, \lambda} = C_{YX}\left(C_{XX} + \lambda Id_{\mathcal{H}_X} \right)^{-1}$. Since $\left(\mu_{i}^{1 / 2} e_{i}\right)_{i \in I}$ is an ONB of $\left(\operatorname{ker} I_{\pi}\right)^{\perp}$, we can complete it with an at most countable basis $\left(\bar{e}_i\right)_{i \in I'}$ of $\operatorname{ker} I_{\pi}$ such that the union of the family forms a basis of $\mathcal{H}_X$. We get a basis of $S_2(\mathcal{H}_X, \mathcal{H}_Y)$ through $\left(d_j \otimes f_i\right)_{i \in I \cup I', j \in J}$ where  $f_i = \mu_{i}^{1 / 2} e_{i}$ if $i \in I$ and $f_i = \bar{e}_i$ if $i \in I'$. By Equation (23) from \cite{fischer2020sobolev}, for $a > 0$ we then have
\begin{equation*}
    \left(C_{XX}+\lambda\right)^{-a}=\sum_{i \in I}\left(\mu_{i}+\lambda\right)^{-a}\left\langle\mu_{i}^{1 / 2} e_{i}, \cdot\right\rangle_{\mathcal{H}_X} \mu_{i}^{1 / 2} e_{i}+\lambda^{-a} \sum_{i \in I'}\left\langle\bar{e}_{i}, \cdot\right\rangle_{\mathcal{H}_X} \bar{e}_{i}.
\end{equation*}

Furthermore, 
\begin{IEEEeqnarray*}{rCl}
C_{YX} &=& \mathbb{E}_{YX}\left[\phi_{Y}(Y) \otimes \phi_{X}(X)\right] \nonumber\\
&=& \mathbb{E}_{X}\left[ \mathbb{E}_{Y|X}\left[\phi_{Y}(Y)\right] \otimes \phi_{X}(X)\right] \nonumber\\
&=& \mathbb{E}_{X}\left[F_{*}(X) \otimes \phi_{X}(X)\right] \nonumber\\
&=& \mathbb{E}_{X}\left[\Psi\left(C_{Y\mid X}\right)(X) \otimes \phi_{X}(X)\right] \nonumber\\
&=& \sum_{i \in I} \sum_{j \in J} \check{a}_{ij} \mathbb{E}_{X}\left[\Psi\left(d_j \otimes [e_{i}]\right)(X) \otimes \phi_{X}(X)\right] \nonumber\\
&=& \sum_{i \in I} \sum_{j \in J} \check{a}_{ij} \mathbb{E}_{X}\left[[e_{i}](X) d_j \otimes \phi_{X}(X)\right],
\end{IEEEeqnarray*}
In the last step we used the explicit form of the isomorphism between $L_2(\pi;\mathcal{H}_Y)$ and $S_2(L_2(\pi), \mathcal{H}_Y)$ mentioned in Remark~\ref{rem:tensor_product}: $\Psi$ is characterized by $\Psi\left(g \otimes f\right) = \left(x \mapsto gf(x) \right)$, for all $g \in \mathcal{H}_Y, f \in L_2(\pi)$. Then, using that $\left([e_{i}]\right)_{i \in I}$ is an ONS in $L_2(\pi),$

\begin{equation*}
    [C_{Y \mid X, \lambda}] = \sum_{i \in I} \sum_{j \in J} \check{a}_{ij}\frac{\mu_i}{\lambda + \mu_i} d_j \otimes [e_{i}],
\end{equation*}
and hence
\begin{equation} \label{eq:decomposition_bias}
    [C_{Y \mid X, \lambda}] - C_{Y \mid X} = -\sum_{i \in I} \sum_{j \in J} \check{a}_{ij}\frac{\lambda}{\lambda + \mu_i} d_j \otimes [e_{i}].
\end{equation}
We are now ready to compute the upper bound. Parseval's identity w.r.t. the ONB $\left(d_j \otimes \mu_{i}^{\gamma / 2}\left[e_{i}\right]\right)_{i \in I, j \in J}$ of $S_2\left([\mathcal{H}]_{X}^{\gamma}, \mathcal{H}_Y\right)$ yields
\begin{equation*}
    \begin{aligned}
    \left\|[C_{Y \mid X, \lambda}] - C_{Y \mid X}\right\|_{S_2\left([\mathcal{H}]_X^{\gamma},\mathcal{H}_{Y}\right)}^{2} &= \left\|\sum_{i \in I} \sum_{j \in J} \check{a}_{ij}\frac{\lambda}{\lambda + \mu_i} d_j \otimes [e_{i}]\right\|_{S_2\left([\mathcal{H}]_X^{\gamma},\mathcal{H}_{Y}\right)}^{2}\\
    &= \sum_{i \in I} \sum_{j \in J} \check{a}_{ij}^2\left(\frac{\lambda}{\lambda + \mu_i}\right)^2\mu_i^{-\gamma}.
    \end{aligned}
\end{equation*}
Next we notice that,
\begin{equation*}
    \begin{aligned}
        \left( \frac{\lambda }{\mu_i + \lambda}\right)^2\mu_i^{-\gamma} &= \left( \frac{\lambda }{\mu_i + \lambda}\right)^2 \mu_i^{-\gamma} \left( \frac{\lambda}{\lambda} \frac{\mu_i+ \lambda}{\mu_i + \lambda}\right)^{\beta-\gamma}\\
        &= \lambda^{\beta-\gamma}\mu_i^{-\beta} \left( \frac{\lambda }{\mu_i + \lambda}\right)^2 \left(\frac{\mu_i}{\mu_i+ \lambda} \right)^{\beta-\gamma} \left( \frac{\mu_i+ \lambda}{ \lambda}\right)^{\beta-\gamma}\\
        &= \lambda^{\beta-\gamma}\mu_i^{-\beta} \left(\frac{\mu_i}{\mu_i+ \lambda} \right)^{\beta-\gamma} \left( \frac{ \lambda}{ \lambda+ \mu_i}\right)^{2-\beta+\gamma}\\
        &\leq \lambda^{\beta-\gamma}\mu_i^{-\beta},
    \end{aligned}
\end{equation*}
where we used $\beta - \gamma \geq 0$ and $2 - \beta + \gamma \geq 0$. Hence,
\begin{equation*}
    \begin{aligned}
    \left\|[C_{Y \mid X, \lambda}] - C_{Y \mid X}\right\|_{S_2\left([\mathcal{H}]_X^{\gamma},\mathcal{H}_{Y}\right)}^{2} &\leq \lambda^{\beta - \gamma}\sum_{i \in I} \sum_{j \in J} \check{a}_{ij}^2\mu_i^{-\beta} \\ 
    &= \lambda^{\beta-\gamma}\left\|C_{Y \mid X}\right\|_{S_2\left([\mathcal{H}]_X^{\beta},\mathcal{H}_{Y}\right)}^{2} 
    \end{aligned}
\end{equation*}
\end{proof}

\subsection{Bounding the Variance}
The proof will require several lemmas in its construction, which we now present. We start with a  lemma that allows to go from the $\gamma$-norm of embedded vector-valued maps to their norm in the original Hilbert-Schmidt space. 

\begin{restatable}{lma}{NormTrasfer}\label{theo:gamma_norm_transfer}
For $0 \leq \gamma \leq 1$ and $F \in \mathcal{G}$ the inequality
\begin{IEEEeqnarray}{rCl}
\left\|[F]\right\|_{\gamma} \leq\left\|CC_{XX}^{\frac{1-\gamma}{2}} \right\|_{S_2(\mathcal{H}_X, \mathcal{H}_Y)} \label{lma:bias_upper}
\end{IEEEeqnarray}
holds, where $C = \bar{\Psi}^{-1}(F) \in S_2(\mathcal{H}_X, \mathcal{H}_Y)$. If, in addition, $\gamma<1$ or $C \perp \mathcal{H}_Y \otimes \operatorname{ker} I_{\pi}$ is satisfied, then the result is an equality.
\end{restatable}
\begin{proof}
Let us fix $F \in \mathcal{G}$, and define $C := \bar{\Psi}^{-1}(F) \in S_2(\mathcal{H}_X, \mathcal{H}_Y)$. Since $\left(\mu_{i}^{1 / 2} e_{i}\right)_{i \in I}$ is an ONB of $\left(\operatorname{ker} I_{\pi}\right)^{\perp}$, we can complete it with a basis $\left(\bar{e}_i\right)_{i \in I'}$ of $\operatorname{ker} I_{\pi}$ such that the union of the family forms a basis of $\mathcal{H}_X$. Let $\left(d_j\right)_{j \in J}$ be a basis of $\mathcal{H}_Y$, we get a basis of $S_2(\mathcal{H}_X, \mathcal{H}_Y)$ through $\left(d_j \otimes f_i\right)_{i \in I \cup I', j \in J}$ where  $f_i = \mu_{i}^{1 / 2} e_{i}$ if $i \in I$ and $f_i = \bar{e}_i$ if $i \in I'$. Then $C$ admits the decomposition $$C=\sum_{i \in I} \sum_{j \in J} a_{ij} d_j \otimes \mu_{i}^{1 / 2} e_{i} + \sum_{i \in I'} \sum_{j \in J} a_{ij} d_j \otimes \bar{e}_{i},$$ where $a_{ij}=\left\langle C,d_j \otimes f_i \right\rangle_{S_2(\mathcal{H}_X, \mathcal{H}_Y)}  = \left\langle C f_i,d_j\right\rangle_{\mathcal{H}_Y}$ for all $i \in I \cup I', j \in J$. Since $$[C]=\sum_{i \in I} \sum_{j \in J} a_{ij} d_j \otimes \mu_{i}^{1 / 2} [e_{i}],$$ with Parseval's identity w.r.t. the ONB $\left(d_j \otimes \mu_{i}^{\gamma / 2}\left[e_{i}\right]\right)_{i \in I, j \in J}$ of $S_2([\mathcal{H}]_X^{\gamma}, \mathcal{H}_Y)$ this yields
$$
\left\|[C]\right\|_{S_2([\mathcal{H}]_X^{\gamma}, \mathcal{H}_Y)}^{2}=\left\|\sum_{i \in I} \sum_{j \in J} a_{ij}\mu_{i}^{\frac{1-\gamma}{2}} d_j \otimes \mu_{i}^{\gamma / 2} [e_{i}]\right\|_{S_2([\mathcal{H}]_X^{\gamma}, \mathcal{H}_Y)}^{2}=\sum_{i \in I} \sum_{j \in J} a_{ij}^2\mu_{i}^{1-\gamma}.
$$
For $\gamma<1$, the spectral decomposition of $C_{XX}$ together with the fact that $\left(d_j \otimes \mu_{i}^{1 / 2} e_{i}\right)_{i \in I, j \in J}$ is an ONS in $S_2(\mathcal{H}_X, \mathcal{H}_Y)$ yields 
\begin{equation}
    \begin{aligned}
    \left\|CC_{XX}^{\frac{1-\gamma}{2}}\right\|_{S_2(\mathcal{H}_X, \mathcal{H}_Y)}^{2} &= \left\|C\sum_{i \in I} \mu_i^{\frac{1-\gamma}{2}} \langle\cdot, \mu_i^{\frac{1}{2}}e_i \rangle_{\mathcal{H}_{X}} \mu_i^{\frac{1}{2}}e_i\right\|_{S_2(\mathcal{H}_X, \mathcal{H}_Y)}^{2} \\ &= \sum_{i \in I}\left\|\sum_{l \in I} \mu_l^{\frac{1-\gamma}{2}} \langle \mu_i^{\frac{1}{2}}e_i, \mu_l^{\frac{1}{2}}e_l \rangle_{\mathcal{H}_{X}} \mu_l^{\frac{1}{2}}Ce_l\right\|_{\mathcal{H}_Y}^{2} + \sum_{i \in I'}\left\|\sum_{l \in I} \mu_l^{\frac{1-\gamma}{2}} \langle \bar{e}_i, \mu_l^{\frac{1}{2}}e_l \rangle_{\mathcal{H}_{X}} \mu_l^{\frac{1}{2}}Ce_l\right\|_{\mathcal{H}_Y}^{2} \\ &= \sum_{i \in I}\left\|\mu_i^{\frac{1-\gamma}{2}} \mu_i^{\frac{1}{2}}Ce_i\right\|_{\mathcal{H}_Y}^{2} \\ &= \sum_{i \in I} \sum_{j \in J}\mu_i^{1-\gamma}\left\langle C\left(\mu_i^{\frac{1}{2}}e_i\right), d_j\right\rangle_{\mathcal{H}_Y}^{2} \\ &= \sum_{i \in I} \sum_{j \in J} a_{ij}^2 \mu_i^{1-\gamma} .
    \end{aligned}
\end{equation}
This proves the claimed equality in the case of $\gamma<1$. For $\gamma=1$, we have $C_{XX}^{\frac{1-\gamma}{2}}=\operatorname{Id}_{\mathcal{H}_X}$ and the Pythagorean theorem together with Parseval's identity yields
\begin{equation}
    \begin{aligned}
    \left\|CC_{XX}^{\frac{1-\gamma}{2}}\right\|_{S_2(\mathcal{H}_X, \mathcal{H}_Y)}^{2} &=\left\|\sum_{i \in I} \sum_{j \in J} a_{ij} d_j \otimes \mu_{i}^{1 / 2} e_{i} + \sum_{i \in I'} \sum_{j \in J} a_{ij} d_j \otimes \bar{e}_{i}\right\|_{S_2(\mathcal{H}_X, \mathcal{H}_Y)}^{2} \\ &=\left\|\sum_{i \in I} \sum_{j \in J} a_{ij} d_j \otimes \mu_{i}^{1 / 2} e_{i}\right\|_{S_2(\mathcal{H}_X, \mathcal{H}_Y)}^{2}+\left\| \sum_{i \in I'} \sum_{j \in J} a_{ij} d_j \otimes \bar{e}_{i}\right\|_{S_2(\mathcal{H}_X, \mathcal{H}_Y)}^{2} \\ &=\sum_{i \in I} \sum_{j \in J} a_{ij}^{2}+ \left\| \sum_{i \in I'} \sum_{j \in J} a_{ij} d_j \otimes \bar{e}_{i}\right\|_{S_2(\mathcal{H}_X, \mathcal{H}_Y)}^{2}
    \end{aligned}
\end{equation}
This gives the claimed equality if $C \perp \mathcal{H}_Y \otimes \operatorname{ker} I_{\pi}$, as well as the claimed inequality for general $C \in S_2(\mathcal{H}_X, \mathcal{H}_Y)$. We conclude with $\|[F]\|_{\gamma} = \|[C]\|_{S_2([\mathcal{H}]_X^{\gamma}, \mathcal{H}_Y)}$ by definition.
\end{proof}

\begin{lma}\label{lma:bias_bound_gamma}
If $F_{*} \in [\mathcal{G}]^{\beta}$ is satisfied for some $0 \leq \beta \leq 2$, then the following bounds is satisfied, for all $\lambda>0$ and $\gamma \geq 0$:
\begin{equation} \label{eq:lemma3_2}
    \left\|\left[F_{\lambda}\right]\right\|_{\gamma}^{2} \leq \left\|F_{*}\right\|_{\min \{\gamma, \beta\}}^{2} \lambda^{-(\gamma-\beta)_{+}}. 
\end{equation}
\end{lma}


\begin{proof}
By Parseval's identity
$$
\left\|\left[F_{\lambda}\right]\right\|_{\gamma}^{2}=\sum_{i \in I} \sum_{j \in J}\left(\frac{\mu_{i}}{\mu_{i}+\lambda}\right)^{2} \mu_{i}^{-\gamma} \check{a}_{ij}^{2} .
$$
where $\check{a}_{ij} = \left\langle C_{Y \mid X} [e_i],d_j\right\rangle_{\mathcal{H}_Y}$ for all $i \in I, j \in J$ as in the proof of Lemma~\ref{lma:new_bias}. In the case of $\gamma \leq \beta$ we estimate the fraction by 1 and then Parseval's identity gives us
$$
\left\|\left[F_{\lambda}\right]\right\|_{\gamma}^{2} \leq \sum_{i \in I} \sum_{j \in J} \mu_{i}^{-\gamma} \check{a}_{ij}^{2}=\left\|F_{*}\right\|_{\gamma}^{2} .
$$
In the case of $\gamma>\beta$,
$$
\left\|\left[F_{\lambda}\right]\right\|_{\gamma}^{2}=\sum_{i \in I} \sum_{j \in J}\left(\frac{\mu_{i}^{1-\frac{\gamma-\beta}{2}}}{\mu_{i}+\lambda}\right)^{2} \mu_{i}^{-\beta} \check{a}_{ij}^{2} \leq \lambda^{-(\gamma-\beta)} \sum_{i \in I} \sum_{j \in J} \mu_{i}^{-\beta} \check{a}_{ij}^{2}=\lambda^{-(\gamma-\beta)}\left\|F_{*}\right\|_{\beta}^{2} ,
$$
where  we used Parseval's identity in the equality and Lemma 25 from \cite{fischer2020sobolev}.
\end{proof}

By \eqref{asst:emb}, the inclusion map $I^{\alpha, \infty}_{\pi}: [\mathcal{H}]_{X}^{\alpha} \hookrightarrow L_{\infty}(\pi)$ has bounded norm $A > 0$ i.e. for $f \in [\mathcal{H}]_{X}^{\alpha}$, $f$ is $\pi-$a.e. bounded and $\|f\|_{\infty} \leq A\|f\|_{\alpha}$. We know show that \eqref{asst:emb} automatically implies that the inclusion operator for $[\mathcal{G}]^{\alpha}$ is bounded.

\begin{lma} \label{lma:_infinite_embedding_v_rkhs}
Under \eqref{asst:emb} the inclusion operator $\mathcal{I}_{\pi}^{\alpha, \infty}: [\mathcal{G}]^{\alpha} \hookrightarrow L_{\infty}(\pi; \mathcal{H}_Y)$ is bounded with operator norm less than or equal to $A$.
\end{lma}
$L_{\infty}(\pi; \mathcal{H}_Y)$ denotes the space of $\mathcal{F}_{E_X} - \mathcal{F}_{\mathcal{H}_Y}$ measurable $\mathcal{H}_Y$-valued functions (gathered by $\pi$-equivalent classes) that are essentially bounded with respect to $\pi$. $L_{\infty}(\pi; \mathcal{H}_Y)$ is endowed with the norm $\|f\|_\infty := \inf \{c \geq 0 : \|f(x)\|_{\mathcal{H}_Y} \leq c \text{ for $\pi$-almost every } x \in E_X\}$.

\begin{proof}
For every $F \in [\mathcal{G}]^{\alpha}$, there is a sequence $a_{ij} \in \ell_2(I \times J)$ such that for $\pi-$almost all $x \in E_X$, \[F(x) = \sum_{i \in I, j \in J} b_{ij} d_j \mu_i^{\alpha/2}[e_i](x)\] where $(d_j)_{j \in J}$ is any orthonormal basis of $\mathcal{H}_Y$ and $\|F\|_{\alpha}^2  = \sum_{i \in I, j \in J} b_{ij}^2.$ We consider $F \in [\mathcal{G}]^{\alpha}$ such that $\sum_{i \in I, j \in J} b_{ij}^2 \leq 1$. For $\pi-$almost all $x \in E_X$,
\begin{equation*} 
    \begin{aligned}
        \|F(x)\|_{\mathcal{H}_Y}^2 &= \left\|\sum_{j \in J}\left(\sum_{i \in I} b_{ij}\mu_i^{\alpha/2}[e_i](x)\right)d_j\right\|_{\mathcal{H}_Y}^2 \\
        &= \sum_{j \in J}\left(\sum_{i \in I} b_{ij}\mu_i^{\alpha/2}[e_i](x)\right)^2 \\
        & \leq \sum_{j \in J} \left(\sum_{i \in I}b_{ij}^2\sum_{i \in I}\mu_i^{\alpha}[e_i](x)^2\right)\\
        &\leq A^2 \sum_{j \in J}\sum_{i \in I} b_{ij}^2 \\ &\leq A^2
    \end{aligned}
\end{equation*}
where we used the Cauchy-Schwarz inequality for each $j \in J$ for the first inequality and a consequence of (EMB) in the second inequality (see Theorem 9 in \cite{fischer2020sobolev}). We therefore conclude $\|\mathcal{I}_{\pi}^{\alpha, \infty}\| \leq A$.
\end{proof}

Combining Lemmas \ref{lma:new_bias}, \ref{lma:bias_bound_gamma} and \ref{lma:_infinite_embedding_v_rkhs} we have the following corollary.
 
\begin{lma} \label{lma:F_l_bounded}
If $F_{*} \in [\mathcal{G}]^{\beta}$ and \eqref{asst:emb} are satisfied for some $0 \leq \beta \leq 2$ and $0< \alpha \leq 1$, then the following bounds are satisfied, for all $0< \lambda \leq 1$:
\begin{equation} \label{eq:lemma4_1} 
    \left\|\left[F_{\lambda}\right] - F_{*}\right\|_{\infty}^{2} \leq \left(\left\|F_{*}\right\|_{\infty} + A\|F_*\|_{\beta}\right)^2 \lambda^{\beta - \alpha},
\end{equation}
\begin{equation} \label{eq:lemma4_2} 
    \left\|\left[F_{\lambda}\right]\right\|_{\infty}^{2} \leq A^2\left\|F_{*}\right\|_{\min \{\alpha, \beta\}}^{2} \lambda^{-(\alpha-\beta)_{+}}.
\end{equation}
In addition, we have $\|F_*\|_{\infty} \leq \kappa_Y$.
\end{lma}
\begin{proof}
For Eq. \ref{eq:lemma4_2}, we use Lemma~\ref{lma:_infinite_embedding_v_rkhs} and Eq. \ref{eq:lemma3_2} in Lemma~\ref{lma:bias_bound_gamma}.
\begin{equation*}
    \left\|\left[F_{\lambda}\right]\right\|_{\infty}^{2} \leq A^2 \left\|\left[F_{\lambda}\right]\right\|_{\alpha}^{2} \leq A^2\left\|F_{*}\right\|_{\min \{\alpha, \beta\}}^{2} \lambda^{-(\alpha-\beta)_{+}}
\end{equation*}

To show Eq. \ref{eq:lemma4_1}, in the case $\beta \leq \alpha$ we use the triangle inequality, Eq. \ref{eq:lemma4_2} and $\lambda \leq 1$ to obtain
\begin{equation*}
    \begin{aligned}
        \left\|\left[F_{\lambda}\right] - F_{*}\right\|_{\infty} &\leq \left\|F_{*}\right\|_{\infty} + \left\|\left[F_{\lambda}\right]\right\|_{\infty} \\ &\leq \left(\left\|F_{*}\right\|_{\infty} + A\left\|F_{*}\right\|_{\beta} \right)\lambda^{-\frac{\alpha-\beta}{2}}
    \end{aligned}
\end{equation*}
In the case $\beta > \alpha$, Eq. \ref{eq:lemma4_1} is a consequence of Lemma~\ref{lma:_infinite_embedding_v_rkhs} and Eq. \ref{eq:lemma3_1} in Lemma~\ref{lma:new_bias} with $\gamma = \alpha$,
\begin{equation*}
    \left\|\left[F_{\lambda}\right] - F_{*}\right\|_{\infty}^{2} \leq A^2 \left\|\left[F_{\lambda}\right] - F_{*}\right\|_{\alpha}^{2} \leq A^2 \left\|F_{*}\right\|_{\beta}^{2} \lambda^{\beta-\alpha} \leq \left(\left\|F_{*}\right\|_{\infty} + A\|F_*\|_{\beta}\right)^2 \lambda^{\beta - \alpha}.
\end{equation*}

$F_*$ belongs to $L_{\infty}(\pi; \mathcal{H}_Y)$. Indeed, for $\pi$-almost all $x \in E_X$ we have 
\begin{IEEEeqnarray*}{rCl}
\|F_*(x)\|_{\mathcal{H}_Y} &= &  \left\|\int_{E_X} \phi_Y(y) p(x,dy) \right\|_{\mathcal{H}_Y} \\
& \leq & \int_{E_X} \left\|\phi_Y(y)  \right\|_{\mathcal{H}_Y}p(x,dy)\\
& \leq & \int_{E_X} \kappa_Y p(x,dy) = \kappa_Y.
\end{IEEEeqnarray*}
\end{proof}

\begin{theo}\label{thm:variance_bound}
Suppose Assumptions $1$ to $3$ and \eqref{asst:emb} with $A>0$ hold. We define
$$
\begin{aligned}
M(\lambda) &= \left\|[F_{\lambda}]-F_*\right\|_{\infty},\\ 
\mathcal{N}(\lambda) &=\operatorname{tr}\left(C_{XX}\left(C_{XX}+\lambda\right)^{-1}\right), \\
Q_{\lambda} &=\max \{M(\lambda), 2\kappa_Y \}, \\
g_{\lambda}& = \log \left( 2e\mathcal{N}(\lambda) \frac{\|C_{XX}\|+\lambda}{\|C_{XX}\|}\right).
\end{aligned}
$$
Then, for $\tau \geq 1$, $\lambda > 0$, $n \geq 8A^{2} \tau g_{\lambda} \lambda^{-\alpha}$ and $\lambda > 0$, with probability $1-4e^{-\tau}$ :
\begin{IEEEeqnarray}{rCl}
\left\|\left[\hat{C}_{Y|X,\lambda} - C_{Y|X,\lambda}\right]\right\|_{S_2\left([\mathcal{H}]_X^{\gamma},\mathcal{H}_{Y}\right)}^2 \leq \frac{576\tau^2}{n\lambda^{\gamma}}\left(4\kappa_Y^{2} \mathcal{N}(\lambda)+\frac{\left\|F_{*}-\left[F_{\lambda}\right]\right\|_{L_{2}(\pi; \mathcal{H}_Y)}^{2}A^{2}}{\lambda^{\alpha}} + \frac{2Q_{\lambda}^2A^2}{n\lambda^{\alpha}}\right) \label{eqn:var_bound}
\end{IEEEeqnarray}
\end{theo}

\begin{proof}
We first decompose the variance term as

\begin{IEEEeqnarray}{rCl}
  &&\hspace{-0.5cm} \left\|\left[\hat{C}_{Y|X,\lambda} - C_{Y|X,\lambda}\right]\right\|_{S_2\left([\mathcal{H}]_X^{\gamma},\mathcal{H}_{Y}\right)}\\
  & =& \left\|\left[\hat{C}_{Y X}\left(\hat{C}_{X X}+\lambda Id\right)^{-1}-C_{Y X}\left(C_{X X}+\lambda Id\right)^{-1}\right]\right\|_{S_2\left([\mathcal{H}]_X^{\gamma},\mathcal{H}_{Y}\right)} \nonumber \\
&\leq& \left\|\left(\hat{C}_{Y X}\left(\hat{C}_{X X}+\lambda Id \right)^{-1}-C_{Y X}\left(C_{X X}+\lambda Id\right)^{-1}\right) C_{X X}^{\frac{1-\gamma}{2}}\right\|_{S_2(\mathcal{H}_X, \mathcal{H}_Y)} \nonumber\\
& \leq& \left\|\left(\hat{C}_{Y X}-C_{Y X}\left(C_{X X}+\lambda Id\right)^{-1}\left(\hat{C}_{X X}+\lambda Id\right)\right)\left(C_{X X}+\lambda Id\right)^{-\frac{1}{2}}\right\|_{S_2(\mathcal{H}_X, \mathcal{H}_Y)} \label{eqn:var1} \\
&&\cdot \left\|\left(C_{X X}+\lambda Id \right)^{\frac{1}{2}}\left(\hat{C}_{X X}+\lambda Id \right)^{-1}\left(C_{X X}+\lambda Id \right)^{\frac{1}{2}}\right\|_{\mathcal{H}_X \to \mathcal{H}_X} \label{eqn:var2} \\
&&\cdot \left\|\left(C_{X X}+\lambda Id \right)^{-\frac{1}{2}}C_{X X}^{\frac{1-\gamma}{2}}\right\|_{\mathcal{H}_X \to \mathcal{H}_X}\label{eqn:var3}
\end{IEEEeqnarray}

where we used Lemma~\ref{theo:gamma_norm_transfer}. Eq.~(\ref{eqn:var2}) is bounded as 
in Theorem $16$ in \cite{fischer2020sobolev},
$$
\left\|\left(C_{X X}+\lambda Id \right)^{\frac{1}{2}}\left(\hat{C}_{X X}+\lambda Id \right)^{-1}\left(C_{X X}+\lambda Id \right)^{\frac{1}{2}}\right\| \leq 3
$$
for $n \geq 8A^{2} \tau g_{\lambda} \lambda^{-\alpha}$ with probability $1-2e^{-\tau}$. For Eq.~(\ref{eqn:var3}) we have, using Lemma 25 from \cite{fischer2020sobolev}
$$
\left\|\left(C_{X X}+\lambda Id \right)^{-\frac{1}{2}}C_{X X}^{\frac{1-\gamma}{2}}\right\| \leq \sqrt{\sup _{i} \frac{\mu_{i}^{1-\gamma}}{\mu_{i}+\lambda}} \leq \lambda^{-\frac{\gamma}{2}}.
$$

Finally for the bound of Eq.~(\ref{eqn:var1}) we show that for $\tau \geq 1$, $\lambda > 0$ and $n \geq 1$ with probability $1-2e^{-\tau}$:
\begin{equation} \label{eq:intermediate_bound}
    \begin{aligned}
&\bigg\|\left(\hat{C}_{Y X}-C_{Y X}\left(C_{X X}+\lambda Id \right)^{-1}(\hat{C}_{X X}+\lambda Id )\right)\left(C_{X X}+\lambda Id \right)^{-\frac{1}{2}}\bigg\|_{S_2(\mathcal{H}_X, \mathcal{H}_Y)}^2 \\ &\leq \frac{64\tau^2}{n}\left(4\kappa_Y^{2} \mathcal{N}(\lambda)+\frac{\left\|F_{*}-\left[F_{\lambda}\right]\right\|_{L_{2}(\pi; \mathcal{H}_Y)}^{2}A^{2}}{\lambda^{\alpha}} + \frac{2Q_{\lambda}^2A^2}{n\lambda^{\alpha}}\right).
    \end{aligned}
\end{equation}

We begin with the decomposition
\begin{IEEEeqnarray*}{rCl}
&&\hat{C}_{Y X}-C_{Y X}\left(C_{X X}+\lambda Id \right)^{-1}\left(\hat{C}_{X X}+\lambda Id \right)\\
& =& \hat{C}_{Y X}-C_{Y X}\left(C_{X X}+\lambda Id \right)^{-1}\left(C_{X X}+\lambda Id +\hat{C}_{X X}-C_{X X}\right) \\
& =& \hat{C}_{Y X}-C_{Y X}+C_{Y X}\left(C_{X X}+\lambda Id \right)^{-1}\left(C_{X X}-\hat{C}_{X X}\right) \\
& =& \hat{C}_{Y X}-C_{Y X}\left(C_{X X}+\lambda Id \right)^{-1} \hat{C}_{X X}-\left(C_{Y X}-C_{Y X}\left(C_{X X}+\lambda Id \right)^{-1} C_{X X}\right) \\
&=& \hat{C}_{Y X}-C_{Y X}\left(C_{X X}+\lambda Id \right)^{-1} \hat{\mathbb{E}}[\phi_X(X) \otimes \phi_X(X)]-\left(C_{YX} - C_{Y X}\left(C_{X X}+\lambda Id \right)^{-1} \mathbb{E}[\phi_X(X) \otimes \phi_X(X)]\right) \\
&=& \hat{\mathbb{E}}\left[\left(\phi_Y(Y)-F_{\lambda}(X)\right) \otimes \phi_X(X)\right]-\mathbb{E}\left[\left(\phi_Y(Y)-F_{\lambda}(X)\right) \otimes \phi_X(X)\right]
\end{IEEEeqnarray*}

where we denote $\hat{\mathbb{E}}[\phi_X(X) \otimes \phi_X(X)] = \frac{1}{n}\sum_i^n \phi_X(x_i) \otimes \phi_X(x_i)$. We wish to apply Theorem~\ref{theo:ope_con_steinwart} with $H = S_2(\mathcal{H}_X, \mathcal{H}_Y)$. We emphasise the difference from \cite{talwai2021sobolev}, where the proof is formulated for bounded linear operators. Consider the random variables $\xi_{0}, \xi_{2}: E_X \times E_Y \rightarrow \mathcal{H}_Y \otimes \mathcal{H}_X$ defined by
$$
\begin{aligned}
&\xi_{0}(x, y):=\left(\phi_Y(y)-F_{\lambda}(x)\right) \otimes \phi_X(x), \\
&\xi_{2}(x, y):=\xi_{0}(x, y)\left(C_{XX}+\lambda Id \right)^{-1 / 2}.
\end{aligned}
$$

Moreover, since our kernels $k_X$ and $k_Y$ are bounded,
\begin{equation*}
    \begin{aligned}
        \left\|\xi_{0}(x, y)\right\|_{S_2(\mathcal{H}_X, \mathcal{H}_Y)} &= \left\|\phi_Y(y)-F_{\lambda}(x)\right\|_{\mathcal{H}_Y}\|\phi_X(x)\|_{\mathcal{H}_X} \\
        &\leq \left\|\phi_Y(y)-F_{\lambda}(x)\right\|_{\mathcal{H}_Y}\kappa_X \\
        &\leq \left(\kappa_Y + \left\|F_{\lambda}(x)\right\|_{\mathcal{H}_Y}\right)\kappa_X, \\
    \end{aligned}
\end{equation*}
and $F_{\lambda}$ is $\pi$-almost surely bounded by Lemma \ref{lma:F_l_bounded}. As a result $\xi_{0}$ is Bochner-integrable.
This yields
$$
\frac{1}{n} \sum_{i=1}^{n}\left(\xi_{2}\left(x_{i}, y_{i}\right)-\mathbb{E} \xi_{2}\right) = \hat{\mathbb{E}} \xi_{2}-\mathbb{E} \xi_{2} = \left(\hat{C}_{Y X}-C_{Y X}\left(C_{X X}+\lambda Id \right)^{-1}\left(\hat{C}_{X X}+\lambda Id \right)\right)\left(C_{X X}+\lambda Id \right)^{-\frac{1}{2}},
$$
and therefore Eq.~(\ref{eqn:var1}) coincides with the left hand side of Bernstein's inequality for $H$-valued random variables (Theorem~\ref{theo:ope_con_steinwart}). Consequently, it remains to bound the $m$-th moment of $\xi_{2}$, for $m \geq 2$,
$$
\mathbb{E}\left\|\xi_{2}\right\|_{S_2(\mathcal{H}_X, \mathcal{H}_Y)}^{m}=\int_{E_X}\left\|\left(C_{XX}+\lambda Id\right)^{-1 / 2} \phi(x)\right\|_{\mathcal{H}_X}^{m} \int_{E_Y}\left\|\phi_Y(y)-F_{\lambda}(x)\right\|_{\mathcal{H}_Y}^{m} p(x, \mathrm{~d} y) \mathrm{d} \pi(x) .
$$
First, we consider the inner integral. Using the triangle inequality and the fact that $\left\|\phi_Y(y)-F_{\lambda}(x)\right\|_{\mathcal{H}_Y} \leq 2\kappa_Y$ almost surely,
$$
\begin{aligned}
\int_{E_Y}\left\|\phi_Y(y)-F_{\lambda}(x)\right\|_{\mathcal{H}_Y}^{m} p(x, \mathrm{~d} y) &\leq 2^{m-1}\left(\left\|\phi_Y(\cdot)-F_{*}(x)\right\|_{L_{m}(p(x, \cdot))}^{m}+\left\|F_{*}(x)-F_{\lambda}(x)\right\|_{\mathcal{H}_Y}^{m}\right) \\
& \leq 2^{2m-1}\kappa_Y^m+2^{m-1}\left\|F_{*}(x)-F_{\lambda}(x)\right\|_{\mathcal{H}_Y}^{m}
\end{aligned}
$$
for $\pi$-almost all $x \in E_X$. If we plug this bound into the outer integral and use the abbreviation $h_{x}:=\left(C_{XX}+\lambda\right)^{-1 / 2} \phi_X(\cdot)$ we get
\begin{equation} \label{eq:m_moment_1}
\begin{aligned}
\mathbb{E}\left\|\xi_{2}\right\|_{S_2(\mathcal{H}_X, \mathcal{H}_Y)}^{m} \leq 2^{2m-1}\kappa_Y^m \int_{E_X}\left\|h_{x}\right\|_{\mathcal{H}_X}^{m} \mathrm{~d} \pi(x) +2^{m-1} \int_{E_X}\left\|h_{x}\right\|_{\mathcal{H}_X}^{m}\left\|F_{*}(x)-F_{\lambda}(x)\right\|_{\mathcal{H}_Y}^{m} \mathrm{~d} \pi(x).
\end{aligned}
\end{equation}
Using Lemma 13~\cite{fischer2020sobolev}, we can bound the first term in Equation \ref{eq:m_moment_1} above by 
$$
\begin{aligned}
2^{2m-1}\kappa_Y^m \int_{E_X}\left\|h_{x}\right\|_{\mathcal{H}_X}^{m} \mathrm{~d} \pi(x) & \leq 2^{2m-1}\kappa_Y^m\left(\frac{A}{\lambda^{\alpha / 2}}\right)^{m-2} \int_{E_X}\left\|h_{x}\right\|_{\mathcal{H}_X}^{2} \mathrm{~d} \pi(x) \\
&=\left(\frac{4 \kappa_Y A}{\lambda^{\alpha / 2}}\right)^{m-2} 8\kappa_Y^{2} \mathcal{N}(\lambda) \\
&\leq \frac{1}{2} m !\left(\frac{2Q_{\lambda} A}{\lambda^{\alpha / 2}}\right)^{m-2} 8 \kappa_Y^{2} \mathcal{N}(\lambda)
\end{aligned}
$$
where we only used $2\kappa_Y \leq Q_{\lambda}$ and $\frac{1}{2}m! \geq 1$ in the last step. Again, using Lemma 13 from~\cite{fischer2020sobolev}, the second term in Equation \eqref{eq:m_moment_1} can be bounded by
$$
\begin{aligned}
& 2^{m-1} \int_{E_X}\left\|h_{x}\right\|_{\mathcal{H}_X}^{m}\left\|F_{*}(x)-F_{\lambda}(x)\right\|_{\mathcal{H}_Y}^{m} \mathrm{~d} \pi(x) \\
\leq & \frac{1}{2}\left(\frac{2A}{\lambda^{\alpha / 2}}\right)^{m}M(\lambda)^{m-2} \int_{E_X}\left\|F_{*}(x) - F_{\lambda}(x)\right\|_{\mathcal{H}_Y}^{2} \mathrm{~d} \pi(x) \\
=& \frac{1}{2}\left(\frac{2AM(\lambda)}{\lambda^{\alpha / 2}}\right)^{m-2}\left\|F_{*}-\left[F_{\lambda}\right]\right\|_{L_{2}(\pi; \mathcal{H}_Y)}^{2} \frac{4A^2}{\lambda^{\alpha}} \\
\leq & \frac{1}{2} m !\left(\frac{2Q_{\lambda}A}{\lambda^{\alpha / 2}}\right)^{m-2}\left\|F_{*}-\left[F_{\lambda}\right]\right\|_{L_{2}(\pi; \mathcal{H}_Y)}^{2} \frac{2A^2}{\lambda^{\alpha}},
\end{aligned}
$$
where we only used $M(\lambda) \leq Q_{\lambda}$ and $2 \leq m$ ! in the last step. Finally, we get
$$
\mathbb{E}\left\|\xi_{2}\right\|_{S_2(\mathcal{H}_X, \mathcal{H}_Y)}^{m} \leq \frac{1}{2} m !\left(\frac{2Q_{\lambda}A}{\lambda^{\alpha / 2}}\right)^{m-2} 2\left(4\kappa_Y^{2} \mathcal{N}(\lambda)+\left\|F_{*}-\left[F_{\lambda}\right]\right\|_{L_{2}(\pi; \mathcal{H}_Y)}^{2} \frac{A^{2}}{\lambda^{\alpha}}\right)
$$
and an application of Bernstein's inequality from Theorem~\ref{theo:ope_con_steinwart} with $$
L=2Q_{\lambda}A\lambda^{-\alpha / 2}, \qquad \sigma^{2}=2\left(4\kappa_Y^{2} \mathcal{N}(\lambda)+\left\|F_{*}-\left[F_{\lambda}\right]\right\|_{L_{2}(\pi; \mathcal{H}_Y)}^{2}A^{2}\lambda^{-\alpha}\right)
$$
yields the bound in Eq.~\ref{eq:intermediate_bound}. Putting all the terms together, we obtain our result.
\end{proof}

\subsection{The CME Learning Rate}
In this section, we aim to establish our upper bound on the learning rate of the conditional mean embedding by combining the learning rates obtained for the bias and variance.

Let us fix some $\tau \geq 1$ and a lower bound $0 < c_1 \leq 1$ with $c_1 \leq \|C_{XX}\|$. We first show that Theorem~\ref{thm:variance_bound} is applicable. To this end, we prove that there is an index bound $n_0 \geq 1$ such that $n \geq 8A^{2} \log \tau g_{\lambda_n} \lambda_n^{-\alpha}$ is satisfied for all $n \geq n_0$. Since $\lambda_n \rightarrow 0$ we choose $n_0' \geq 1$ such that $\lambda_n \leq c_1 \leq \min\{1, \|C_{XX}\|\}$ for all $n \geq n_0'$. We get for $n \geq n_0'$, 
$$
\begin{aligned}
\frac{8A^{2} \tau g_{\lambda_{n}} \lambda_{n}^{-\alpha}}{n} &=\frac{8A^{2} \tau \lambda_{n}^{-\alpha}}{n} \cdot \log \left(2 e \mathcal{N}\left(\lambda_{n}\right) \frac{\left\|C_{XX}\right\|+\lambda_{n}}{\left\|C_{XX}\right\|}\right) \\
& \leq \frac{8A^{2} \tau \lambda_{n}^{-\alpha}}{n} \cdot \log \left(4 c e \lambda_{n}^{-p}\right) \\
&=8A^{2}\tau \left(\frac{\log\left(4 c e\right) \lambda_{n}^{-\alpha}}{n}+\frac{p \lambda_{n}^{-\alpha} \log \lambda_{n}^{-1}}{n}\right)
\end{aligned}
$$
where the second step uses Lemma~\ref{lma:effective_dim}. Hence, it is enough to show $\frac{\log(\lambda_n^{-1})}{n\lambda_n^{\alpha}} \rightarrow 0$. We consider the cases $\beta + p \leq \alpha$ and $\beta + p > \alpha$.
\begin{itemize}
    \item $\beta + p \leq \alpha.$ By substituting that $\lambda_n = \Theta \left(\left(\frac{n}{\log^r n}\right)^{-\frac{1}{\alpha}}\right)$ for some $r > 1$ we have \[\frac{\lambda_{n}^{-\alpha} \log \lambda_{n}^{-1}}{n}=\Theta\left(\frac{\log(n)}{n}\frac{n}{\log^r(n)}\right)  = \Theta \left( \frac{1}{\log^{r-1}(n)}\right) \rightarrow 0, \text{ as }~n \rightarrow \infty.\]
    \item $\beta + p > \alpha.$ By substituting that $\lambda_n = \Theta \left(n^{-\frac{1}{\beta + p}}\right)$ and using $1 - \frac{\alpha}{\beta + p} > 0$ we have \[\frac{\lambda_{n}^{-\alpha} \log \lambda_{n}^{-1}}{n}=\Theta\left(\frac{\log(n)}{n}n^{\frac{\alpha}{\beta + p}}\right)  = \Theta \left( \frac{\log(n)}{n^{1-\frac{\alpha}{\beta + p}}}\right) \rightarrow 0, \text{ as }~n \rightarrow \infty.\]
\end{itemize}
Consequently, there is a $n_0 \geq n_0'$ with $n \geq 8A^{2} \log \tau g_{\lambda_n} \lambda_n^{-\alpha}$ for all $n \geq n_0$. Moreover, $n_0$ just depends on $\lambda_n, c, c_1, \tau, A$, and on the parameters $\alpha,p$. 

Let $n \geq n_0$ be fixed. By Theorem~\ref{thm:variance_bound}, we have 
\begin{IEEEeqnarray*}{rCl}
\left\|\left[\hat{C}_{Y|X,\lambda} - C_{Y|X,\lambda}\right]\right\|_{S_2([\mathcal{H}]_X^{\gamma}, \mathcal{H}_Y)}^2 \leq \frac{576\tau^2}{n\lambda_n^{\gamma}}\left(4\kappa_Y^{2} \mathcal{N}(\lambda_n)+\frac{\left\|F_{*}-\left[F_{\lambda}\right]\right\|_{L_{2}(\pi; \mathcal{H}_Y)}^{2}A^{2}}{\lambda_n^{\alpha}} + \frac{2Q_{\lambda_n}^2A^2}{n\lambda_n^{\alpha}}\right).
\end{IEEEeqnarray*}
Using Lemma~\ref{lma:effective_dim}, Lemma~\ref{lma:bias_bound_gamma} with $\gamma=0$, we have
\begin{IEEEeqnarray}{rCl}
\left\|[\hat{C}_{Y \mid X}-C_{Y \mid X}^{\lambda}]\right\|_{S_2([\mathcal{H}]_X^{\gamma}, \mathcal{H}_Y)}^2 \leq \frac{576\tau^2}{n\lambda_n^{\gamma}}\left(4\kappa_Y^{2} c\lambda_n^{-p}+A^2\|F_*\|_{\beta}^2\lambda_n^{\beta-\alpha} + \frac{2Q_{\lambda_n}^2A^2}{n\lambda_n^{\alpha}}\right)\nonumber
\end{IEEEeqnarray}
For the last term, using the definition of $Q_{\lambda}$ in Theorem~\ref{thm:variance_bound} with Lemma~\ref{lma:F_l_bounded} and $\lambda_n \leq 1,$ we get 
\begin{IEEEeqnarray}{rCl}
Q_{\lambda_n}^2 &=& \max\{(2\kappa_Y)^2, \left\|[F_{\lambda}]-F_*\right\|_{\infty}^2\} \nonumber \\ &\leq& \max\left\{(2\kappa_Y)^2, \left(\left\|F_{*}\right\|_{\infty} + A \|F_*\|_{\beta}\right)^2 \lambda_n^{-(\alpha-\beta)}\right\} \nonumber \\
&\leq& K_0 \lambda_n^{-(\alpha-\beta)_+},
\end{IEEEeqnarray}
where $K_0 := \max\left\{(2\kappa_Y)^2, \left(B_{\infty} + A\|F_*\|_{\beta}\right)^2\right\}$. Thus, 
\begin{IEEEeqnarray}{rCl}
\left\|[\hat{C}_{Y \mid X}-C_{Y \mid X}^{\lambda}]\right\|_{S_2([\mathcal{H}]_X^{\gamma}, \mathcal{H}_Y)}^2 \leq \frac{576\tau^2}{n\lambda_n^{\gamma}}\left(4\kappa_Y^{2} c\lambda_n^{-p}+A^2\|F_*\|_{\beta}^2\lambda_n^{\beta-\alpha} + 2A^2K_0\frac{1}{n\lambda_n^{\alpha +(\alpha-\beta)_{+}}}\right)\nonumber.
\end{IEEEeqnarray}

For the first and second terms in the bracket, we use again the fact that $\lambda_n \leq 1,$ and get
\[4c\kappa_Y^2\lambda_n^{-p} + A^2\|F_*\|_{\beta}^2 \lambda_n^{-(\alpha-\beta)} \leq \left(4c\kappa_Y^2 + A^2\|F_*\|_{\beta}^2\right)\max\{\lambda_n^{-p}, \lambda_n^{-(\alpha-\beta)}\} \leq K_1 \lambda_n^{-\max\{p,\alpha-\beta\}}\]
with $K_1 := 4c\kappa_Y^2 + A^2\|F_*\|_{\beta}^2$. We now have 
\begin{IEEEeqnarray}{rCl}
\left\|[\hat{C}_{Y \mid X}-C_{Y \mid X}^{\lambda}]\right\|_{S_2([\mathcal{H}]_X^{\gamma}, \mathcal{H}_Y)}^2 &\leq & \frac{576\tau^2}{n\lambda_n^{\gamma}}\left(K_1 \lambda_n^{-\max\{p,\alpha-\beta\}} + 2A^2K_0\frac{1}{n\lambda_n^{\alpha +(\alpha-\beta)_{+}}}\right)\nonumber\\
& = & \frac{576\tau^2}{n\lambda_n^{\gamma+\max\{p,\alpha-\beta\}}}\left(K_1 + 2A^2K_0\frac{1}{n\lambda_n^{\alpha +(\alpha-\beta)_{+}-\max\{p,\alpha-\beta\}}}\right).\nonumber
\end{IEEEeqnarray}
Again, we treat the cases $\beta + p \leq \alpha$ and $\beta + p > \alpha$ separately.
\begin{itemize}
    \item $\beta + p \leq \alpha$. In this case we have \[\alpha +(\alpha-\beta)_{+}-\max\{p,\alpha-\beta\} = \alpha.\] Since \(\lambda_n = \Theta \left(\left(\frac{n}{\log^r n}\right)^{-\frac{1}{\alpha}}\right),\) for some $r > 1$ we therefore have \[\frac{1}{n\lambda_n^{\alpha +(\alpha-\beta)_{+}-\max\{p,\alpha-\beta\}}} = \frac{1}{n\lambda^{\alpha}} = \Theta\left( \frac{1}{\log^rn}\right).\]  
    \item $\beta + p > \alpha$. We have $p > \alpha - \beta$ and \(\lambda_n = \Theta \left(n^{-\frac{1}{\beta + p}}\right),\) and hence \[\frac{1}{n\lambda_n^{\alpha +(\alpha-\beta)_{+}-\max\{p,\alpha-\beta\}}} = \frac{1}{n\lambda_n^{\alpha+(\alpha-\beta)_{+} -p}} = \Theta\left(\left(\frac{1}{n}\right)^{1-\frac{\alpha+(\alpha-\beta)_{+}-p}{\beta+p}}\right).\] Using $p > \alpha - \beta$ again gives us $$1-\frac{\alpha+(\alpha-\beta)_{+}-p}{\beta+p} = \frac{2p - (\alpha-\beta)_{+} - (\alpha - \beta)}{\beta+p} > 0.$$
\end{itemize}
As such, there is a constant $K_2 > 0$ with
\[\left\|[\hat{F}_{\lambda} - F_{\lambda}]\right\|_{\gamma}^2  =\left\|[\hat{C}_{Y \mid X}-C_{Y \mid X}^{\lambda}]\right\|_{S_2([\mathcal{H}]_X^{\gamma}, \mathcal{H}_Y)}^2 \leq 576 \frac{\tau^2}{n\lambda_n^{\gamma+\max\{p,\alpha-\beta\}}}\left(K_1 + 2A^2K_0K_2\right)\] for all $n \geq n_0.$ Defining $K_3 := 576(K_1 + 2A^2K_0K_2)$, and using the bias-variance splitting from Eq.~(\ref{eqn:risk_decom}) and Lemma~\ref{lma:new_bias}, we have
\begin{IEEEeqnarray}{rCl}
\left\|[\hat{F}_{\lambda}]-F_*\right\|_{\gamma}^2 &\leq & 2\|C_{Y|X}\|_{S_2([\mathcal{H}]_X^{\beta}, \mathcal{H}_Y)}^2 \lambda_n^{\beta-\gamma} + 2K_3\frac{\tau^2}{n\lambda_n^{\gamma+\max\{p,\alpha-\beta\}}}\nonumber\\
&\leq & \tau^2 \lambda_n^{\beta-\gamma}\left(2\|C_{Y|X}\|_{S_2([\mathcal{H}]_X^{\beta}, \mathcal{H}_Y)}^2+2K_3 \frac{1}{n\lambda_n^{\max\{\beta+p,\alpha\}}}\right), \nonumber
\end{IEEEeqnarray}
where we used $\tau \geq 1$ and $\lambda_n \leq 1$. Since in both cases $\beta + p \leq \alpha$ and $\beta + p > \alpha$, $\lambda_n \succcurlyeq n^{-1/\max\{\alpha, \beta + p\}}$ there is some constant $K>0$ such that \[\left\|[\hat{F}_{\lambda}]-F_*\right\|_{\gamma}^2 \leq  \tau^2 K \lambda_n^{\beta - \gamma}\] for all $n \geq n_0$.

\section{Proof of Theorem \ref{theo:lower_bound}}
In this section, we establish  a lower bound on the learning rate for the empirical conditional mean embedding. To this end, we build on the lower bound for  kernel ridge regression for real-valued outputs in \cite{fischer2020sobolev}, and for finite dimensional vector-valued outputs in \cite{caponnetto2007optimal,grunewalder2012conditional, devore2004mathematical}. The usual approach to build such lower bounds is to construct a family of distributions on the data space and to control the Kullback-Leibler divergence between each pair of distributions. We cannot directly adapt the proofs of \cite{caponnetto2007optimal,grunewalder2012conditional, devore2004mathematical} and \cite{fischer2020sobolev}, however, since both \cite{grunewalder2012conditional} and \cite{fischer2020sobolev} requires the output space to be finite dimensional, which is not the case in our setting. In addition, \cite{fischer2020sobolev} builds a Gaussian distribution for $Y$ conditioned on $X$.
It would be a challenge to build a distribution on $E_X \times E_Y$ so as to attain the required Gaussian conditional distribution in feature space $\mathcal{H}_Y,$ however.

Our novelty in obtaining the lower bound is to reduce the infinite dimensional learning to a specially designed scalar regression  problem. We show that the learning risk is lower bounded by the learning problem evaluated at a particular point (Eq.~\eqref{eq:bound_scalar_bis}), which can be seen as the risk of a scalar-valued regression problem. This effectively allows us to derive the lower bound exploiting proof techniques from \cite{caponnetto2007optimal, fischer2020sobolev}.

We start by noticing that for any $F \in L_2(\pi; \mathcal{H}_Y)$ and $a \in E_Y$,
\begin{equation} \label{eq:bound_scalar}  
\begin{aligned}
    \int_{E_X} \left(\langle F(x), \phi_Y(a) \rangle_{\mathcal{H}_Y}  - \langle F_*(x), \phi_Y(a) \rangle_{\mathcal{H}_Y} \right)^2 d\pi(x) &\leq \int_{E_X} \|F(x) - F_{*}(x)\|_{\mathcal{H}_Y}^2 \|\phi_Y(a)\|^2_{\mathcal{H}_Y} d\pi(x) \\ &\leq  \kappa_Y^2 \|F - F_{*}\|^2_{L_2(\pi; \mathcal{H}_Y)}.
\end{aligned}
\end{equation}
Moreover, by Lemma~\ref{lma:gamma_norm_reduction}, the inequality holds for general $\gamma$-norm (which implies the previous equation, setting $\gamma=0$), 
\begin{IEEEeqnarray}{rCl}
\|\langle F, \phi_Y(a)\rangle_{\mathcal{H}_Y}-\langle F_*, \phi_Y(a)\rangle_{\mathcal{H}_Y}\|_{\gamma} \leq \kappa_Y\|F-F_*\|_{\gamma}. \label{eq:bound_scalar_gamma}
\end{IEEEeqnarray}

\begin{lma}\label{lma:gamma_norm_reduction}
Let $\gamma \geq 0$, for any $F \in [\mathcal{G}]^{\gamma}$ and $a \in E_Y$, we have \[\|\langle F, \phi_Y(a)\rangle_{\mathcal{H}_Y}\|_{\gamma} \leq \kappa_Y \|F\|_{\gamma}.\]
\end{lma}

\begin{proof}
The case where $\gamma = 0$ is already proved in Eq.~(\ref{eq:bound_scalar}). We now let $\gamma > 0$. Recall $\{d_j\}_{j\in J}$ and $\{\mu_i^{\gamma/2}[e_i]\}_{i \in I}$ are the orthonormal basis of $\mathcal{H}_Y$ and $[\mathcal{H}]_{X}^{\gamma}$, since $F \in [\mathcal{G}]^{\gamma}$, we can write $F$ as \[F = \sum_{i,j} a_{ij} d_j \mu_i^{\gamma/2}[e_i].\] Therefore, we have \[\langle F(.), \phi_Y(a)\rangle_{\mathcal{H}_Y} = \sum_{i,j} a_{ij}d_j(a)\mu_i^{\gamma/2}[e_i](.).\] $\langle F(.), \phi_Y(a)\rangle_{\mathcal{H}_Y}$ is a function in $[\mathcal{H}]_X^{\gamma}$ as 
\begin{IEEEeqnarray*}{rCl}
\|\langle F(.), \phi_Y(a)\rangle_{\mathcal{H}_Y}\|_{\gamma}^2 & =& \sum_i \left(\sum_j a_{ij} d_j(a)\right)^2\\
& \leq & \sum_i \sum_j a_{ij}^2 \sum_j d_j^2(a) \\
&=& k_Y(a,a) \sum_{i,j}a_{ij}^2 \\
&\leq & \kappa_Y^2 \|F\|_{\gamma}^2 < +\infty,
\end{IEEEeqnarray*}
where for the second step, we used Cauchy-Schwartz inequality. The third step is due to Parseval's identity since $\{d_j\}_{j \in J}$ is an orthonormal basis of $\mathcal{H}_Y$. 
\end{proof}

We now express the l.h.s as the risk of a scalar-valued regression. Consider a distribution $P$ on $E_X \times E_Y$ that factorizes as $P(x,y) = p(y \mid x)\pi(x)$ for all $(x,y) \in E_X \times E_Y$. For all $x \in E_X$, $p(\cdot \mid x)$ defines a probability distribution on $E_Y$. We fix an element $a \in E_Y$ and define $E_Y^a := k_Y(E_Y,a) = \{y_a \in \mathbb{R} \mid y_a = k_Y(y,a), y \in E_Y\}$. Consider the joint distribution $P_a$ on $E_X \times E_Y^a$ such that

\begin{equation}
\begin{aligned}
    p_a(. \mid x) &:= \left(k_Y(\cdot,a)\right)_{\#}p(\cdot \mid x) \\
    P_a(x, y_a) &:= p_a(y_a \mid x)\pi(x), \quad (x,y_a) \in E_X \times E_Y^a
\end{aligned}
\end{equation}

where $\#$ denotes the push-forward operation. For a dataset $D = \{(x_i, y_i)\}_{i=1}^n \in \left(E_X \times E_Y\right)^n$ where the data are i.i.d from $P$, the dataset $D_a = \{(x_i, y_{a.i})\}_{i=1}^n \in \left(E_X \times E_Y^a \right)^n \subseteq \left(E_X \times \mathbb{R} \right)^n$ where $y_{a.i} := k_Y(y_i,a)$ for all $i=1,\ldots,n$ is i.i.d from $P_a$. Note that $p_a(\cdot \mid x)$ is a probability distribution on $\mathbb{R}$ for all $x$ supported by $\pi$. By definition of the push-forward operator, the Bayes predictor associated to the joint distribution $P_a$ is 
\begin{equation} \label{eq:bayes_scalar}
\begin{aligned}
    f_{a,*}(x) = \int_{\mathbb{R}} y_a dp_a(y_a \mid x) &= \int_{E_Y} k_Y(y,a) dp(y \mid x) \\ &= \left\langle \int_{E_Y} \phi_Y(y) dp(y \mid x), \phi_Y(a)  \right\rangle_{\mathcal{H}_Y} \\ &= \langle F_*(x), \phi_Y(a) \rangle_{\mathcal{H}_Y}
\end{aligned}
\end{equation}
where $F_*$ is the $\mathcal{H}_Y$-valued conditional mean embedding associated to $P$. Therefore, plugging Eq.~(\ref{eq:bayes_scalar}) in Eq.~(\ref{eq:bound_scalar}) we obtain that for any learning method $D \rightarrow \hat{F}_{D} \in \left(\mathcal{H}_Y\right)^{E_X}$

\begin{equation}
\|[\hat{F}_D] - F_{*}\|_{L_2(\pi; \mathcal{H}_Y)} \geq \kappa_Y^{-1}\|[\hat{f}_{D_a}] - f_{a.*}\|_{L_2(\pi)} \label{eq:bound_scalar_bis}
\end{equation}
where $\hat{f}_{D_a}(.) := \langle \hat{F}_D(.), \phi_Y(a)\rangle_{\mathcal{H}_Y}$. The r.h.s is the error measured in $L_2$-norm of the learning method $D_a \rightarrow \hat{f}_{D_a} \in \mathbb{R}^{E_X}$ on the scalar-regression learning problem associated to $D_a$. 

To derive a lower bound on the r.h.s in Eq.~\ref{eq:bound_scalar_bis}, the strategy is to define a conditional distribution $p_a(. \mid x)$ on $E_Y^a$, $x\in E_X$, that is difficult to learn. As $E_Y^a$ is a bounded subset of $\mathbb{R}$, we cannot directly exploit the Gaussian conditional distributions used in \cite{fischer2020sobolev}. Indeed, for all $y \in E_Y$, $|k_Y(y,a)| \leq \kappa_Y^2$. Instead, we suggest to swap the Gaussian conditional distributions used in \cite{fischer2020sobolev} with the discrete conditional distributions used in \cite{caponnetto2007optimal,grunewalder2012conditional, devore2004mathematical}.

We will need the following Lemma that corresponds to Lemma 19 and Lemma 23 and Equation (55) in \cite{fischer2020sobolev}.

\begin{lma} \label{theo:function_construction} Let $k_X$ be a kernel on $E_X$ such that Assumptions $1$-$3$ hold and $\pi$ be a probability distribution on $E_X$ such that \eqref{asst:evd+} and \eqref{asst:emb} are satisfied for some $0<p \leq \alpha \leq 1$. Then, for all parameters $0<\beta \leq 2$, $0 \leq \gamma \leq 1$ with $\gamma < \beta$ and all constants $\bar{B}, B_{\infty}>0$, there exist constants $0<\epsilon_{0} \leq 1$ and $C_0, C>0$ such that the following statement is satisfied: for all $0<\epsilon \leq \epsilon_{0}$ there is an $M_{\epsilon} \geq 1$ with
\begin{IEEEeqnarray}{rCl}
2^{C_0 \epsilon^{-u}} \leq M_{\epsilon} \leq 2^{3 C_0 \epsilon^{-u}} \label{eqn:constant_constraint}
\end{IEEEeqnarray}
where $u:=\frac{p}{\max \{\alpha, \beta\} - \gamma}$ and functions $f_1, \ldots, f_{M_{\epsilon}}$ such that $f_{i} \in [\mathcal{H}]_X^{\beta}$, $\|f_{i}\|_{\beta} \leq \bar{B}$, $\|f_{i}\|_{L_{\infty}(\pi)} \leq B_{\infty}$, and
\begin{IEEEeqnarray}{rCl}
&& \left\|f_{i}-f_{j}\right\|^{2}_{\gamma} \geq 4\epsilon \\ &&\left\|f_{i}-f_{j}\right\|^{2}_{L_2(\pi)} \leq 32C^{\gamma}\epsilon m^{-\gamma/p},\label{eqn:F_divergence}
\end{IEEEeqnarray}
for all $i, j \in\left\{0, \ldots, M_{\varepsilon}\right\}$ with $i \neq j$ where $m$ comes from Lemma 23 in \cite{fischer2020sobolev}.
\end{lma}

We now combine Lemma~\ref{theo:function_construction} with the conditional distributions introduced in \cite{caponnetto2007optimal, devore2004mathematical}. 

\begin{lma}\label{theo:lower_construction} Under the notations and assumptions of Lemma~\ref{theo:function_construction} there are probability measures $P_{a,0}, P_{a,1}, \ldots, P_{a,M_{\epsilon}}$ on $E_X \times E_Y^a$ each with marginal distribution $\pi$ on $E_X$, for which the Bayes estimators $f^a_{*,i}$, $i=1,\ldots,M_{\epsilon}$ satisfy $f^a_{*,i} = f_i + r$, $r \in \mathbb{R}$ where the $f_i's$ have been introduced in Lemma~\ref{theo:function_construction}.
Furthermore, 

\begin{IEEEeqnarray*}{rCl}
KL(P_{a,i},P_{a,j}) &\leq& 40B_{\infty}^2C^{\gamma}\epsilon m^{-\gamma/p}
\end{IEEEeqnarray*}
for all $i, j \in\left\{0, \ldots, M_{\varepsilon}\right\}$ with $i \neq j$, where $KL$ denotes the Kullback-Leibler divergence and $C,B_{\infty}$ come from Lemma~\ref{theo:function_construction}.
\end{lma}

\begin{proof} For all $i=1,\ldots,M_{\epsilon}$, recall that $\|f_i\|_{L_{\infty}(\pi)} \leq B_{\infty}$. Pick any point $r \in \mathbb{R}$ such that $r-L$ and $r+L$ belong to $E_Y^a$ where $L := 1.5B_{\infty}$. Define the joint distribution $P_{a,i}(x,y_a) = p_{a,i}(y_a \mid x) \pi(x)$ where

\begin{equation} \label{eq:cond_dist}
p_{a,i}(y_a \mid x) = \frac{1}{2L}\left\{(L-f_{i}(x))\delta_{r-L}(\{y_a\}) + (L+f_{i}(x))\delta_{r+L}(\{y_a\}) \right\}, \quad y_a \in E_Y^a
\end{equation} 
where $\delta_{r\pm L}$ is a Dirac measure on $E_Y^a$ at point $r \pm L$. $p_{a,i}(. \mid x)$ defines a probability distribution on $\mathbb{R}$ such that \[f_{*,i}^a(x) = \int_{\mathbb{R}} y_a dp(y_a\mid x) = \frac{1}{2L}\left\{(L-f_{i}(x))(r-L) + (L+f_{i}(x))(r+L) \right\} = r+f_{i}(x).\]

We now investigate the KL divergence between $P_{a,i}$ and $P_{a,j}$. The proof is the same as in Proposition 4 of \cite{caponnetto2007optimal} and Lemma 3.2 of \cite{devore2004mathematical}. We first note that 
\begin{IEEEeqnarray*}{rCl}
\log \left(\frac{L\pm f_{i}(x)}{L\pm f_{j}(x)} \right) & = & \log \left(1 \pm \frac{f_{i}(x) -f_{j}(x)}{L\pm f_{j}(x)} \right)\\
&\leq & \pm \frac{f_{j}(x) -f_{j}(x)}{L\pm f_{j}(x)}.
\end{IEEEeqnarray*}
Therefore, we can bound the KL divergence between $P_{a,i}$ and $P_{a,j}$ as 
\begin{IEEEeqnarray*}{rCl}
KL(P_{a,i},P_{a,j})  &\leq & \frac{1}{2L}\int_{E_X} \frac{f_{i}(x) -f_{j}(x)}{L + f_{j}(x)} (L + f_{i}(x)) - \frac{f_{i}(x) -f_{j}(x)}{L - f_{j}(x)}(L-f_{i}(x)) d\pi(x)\\
& =& \frac{1}{2L}\int_{E_X} \frac{f_{i}(x) -f_{j}(x)}{L + f_{j}(x)} (L + f_{j}(x) +f_{i}(x) -f_{j}(x))\\
&& - \frac{f_{i}(x) -f_{j}(x)}{L - f_{j}(x)}(L-f_{j}(x) + f_{j}(x)-f_{i}(x)) d\pi(x)\\
& =& \frac{1}{2L}\int_{E_X} \frac{(f_{i}(x) -f_{j}(x))^2}{L + f_{j}(x)} + \frac{(f_{i}(x) -f_{j}(x))^2}{L - f_{j}(x)}d\pi(x)\\
& =& \int_{E_X} \frac{(f_{i}(x) -f_{j}(x))^2}{L^2 - f_{i}^2(x)} d\pi(x) \leq 1.25B_{\infty}^2\|f_{i} - f_{j}\|_{L_2(\pi)}^2 \\ &\leq& 40B_{\infty}^2C^{\gamma}\epsilon m^{-\gamma/p}.
\end{IEEEeqnarray*}
\end{proof}

Combining Lemma~\ref{theo:function_construction} and Lemma~\ref{theo:lower_construction} allows us to derive a lower bound on the scalar-valued regression associated to $D_a$. The proof of the following Theorem is a consequence of Theorem 20, Lemma 19 and Theorem 2 in \cite{fischer2020sobolev}.

\begin{theo}\label{theo:lower_bound_scalar_final}
Under the notations and assumptions of Lemma~\ref{theo:function_construction} there exists constants $K_0, K, s > 0$ such that for all learning methods $D_{a} \rightarrow \hat{f}_{D_a}$, all $\tau > 0$, and all sufficiently large $n \geq 1$ there is a distribution $P_a$ defined on $E_X \times E_Y^a$ used to sample $D_a$, with marginal distribution $\pi$ on $E_X$ such that $f_{a,*} \in [\mathcal{H}]_X^{\beta}$, and $\|f_{a,*}\|_{\infty} \leq B_{\infty}$, and with probability not less than $1-K_0\tau^{1/s}$, \[\|[\hat{f}_{D_a}] - f_{a,*}\|^2_{\gamma} \geq \tau^2 K n^{-\frac{\max\{\alpha,\beta\}-\gamma}{\max\{\alpha, \beta\}+p}} .\]
\end{theo}

We now use Theorem~\ref{theo:lower_bound_scalar_final} in conjunction with Eq.~(\ref{eq:bound_scalar_bis}) to prove Theorem~\ref{theo:lower_bound}.

\begin{proof}[Proof of Theorem~\ref{theo:lower_bound}]
The conditional distribution used in the proof of Lemma~\ref{theo:lower_construction} Eq.~(\ref{eq:cond_dist}) to obtain a lower bound on the scalar-valued regression risk  takes the form 

\begin{equation} 
p_a(y_a \mid x) = \frac{1}{2L}\left\{(L-f(x))\delta_{r-L}(\{y_a\}) + (L+f(x))\delta_{r+L}(\{y_a\}) \right\}1_{y_a \in \phi_Y(E_Y)(a)}, \qquad y_a \in \mathbb{R}
\end{equation} 
with $L = 1.25B_{\infty}$, $f \in [\mathcal{H}]_X^{\beta}$, $\left\|f_{\omega}\right\|_{\beta} \leq \bar{B}$ and $\left\|f_{\omega}\right\|_{L_{\infty}(\pi)} \leq B_{\infty}$. Since $r \pm L \in E_Y^a$, there exists $y_{\pm} \in E_Y$ such that $\phi_y(y_{\pm})(a) = r \pm L$. Therefore, for all $x \in E_X$,
\begin{equation} 
p(y \mid x) = \frac{1}{2L}\left\{(L-f(x))\delta_{y_{-}}(\{y\}) + (L+f(x))\delta_{y_{+}}(\{y\}) \right\}, \qquad y \in E_Y
\end{equation}
defines a family of contional distributions on $E_Y$ such that $p_a(. \mid x) = \left(k_Y(\cdot,a)\right)_{\#}p(\cdot \mid x)$. For the joint distribution $p(x,y) = p(y \mid x)\pi(x)$ the conditional mean embedding is
\begin{equation}
\begin{aligned}
    F_*(x) &= \int_{E_Y} \phi_Y(y)dp(y \mid x) \\
    &= \frac{1}{2L}\left\{(L-f(x))\phi_Y(y_{-}) + (L+f(x))\phi_Y(y_{+}) \right\}
    \end{aligned}
\end{equation}
As a result, we have
\begin{IEEEeqnarray*}{rCl}
\|F_*\|_{\beta} &= & \frac{1}{2L} \left\|\left\{(L-f)\otimes \phi_Y(y_{-}) + (L+f)\otimes \phi_Y(y_{+}) \right\} \right\|_{\beta} \\
& \leq & \frac{1}{2L} \left(\left\|(L-f)\otimes \phi_Y(y_{-}) \right\|_{\beta}  + \left\| (L+f)\otimes\phi_Y(y_{+}) \right\|_{\beta} \right)\\
& =& \frac{1}{2L} \left(\left\|L-f\right\|_{\beta} \|\phi_Y(y_{-})\|_{\mathcal{H}_Y}  + \left\|L+f\right\|_{\beta} \|\phi_Y(y_{+})\|_{\mathcal{H}_Y} \right)\\
&\leq& \frac{\kappa_Y}{2L} \left(\left\|L-f\right\|_{\beta}  + \left\|L+f\right\|_{\beta} \right)\\
&=& \kappa_Y + \frac{\kappa_Y}{L} \|f\|_{\beta} < +\infty,
\end{IEEEeqnarray*}
where the third step follows from Definition~\ref{def:inter_ope_norm}, the fourth step is due to the boundedness of kernel $k_Y$ and the second last step follows from Eq.~(\ref{eqn:cons_f_norm}). We conclude by combining Theorem~\ref{theo:lower_bound_scalar_final} with Eq.~(\ref{eq:bound_scalar_bis}).
\end{proof}

\section{Auxiliary Results}\label{sec:auxiliary}
The following lemma is from \cite{fischer2020sobolev}.
\begin{lma}\label{theo:h_bound}
Under \eqref{asst:emb} we have
$$
\left\|\left(C_{X X}+\lambda Id_{\mathcal{H}_X} \right)^{-\frac{1}{2}} k(X, \cdot)\right\|_{\mathcal{H}_X} \leq A\lambda^{-\frac{\alpha}{2}}.
$$
\end{lma}

The following Theorem is from \cite[Theorem $26$]{fischer2020sobolev}.
\begin{theo}\label{theo:ope_con_steinwart}
Bernstein's Inequality. Let $(\Omega, \mathcal{B}, P)$ be a probability space, $H$ be a separable Hilbert space, and $\xi: \Omega \rightarrow H$ be a random variable with
$$
\mathbb{E}_{P}\|\xi\|_{H}^{m} \leq \frac{1}{2} m ! \sigma^{2} L^{m-2}
$$
for all $m \geq 2$. Then, for $\tau \geq 1$ and $n \geq 1$, the following concentration inequality is satisfied
$$
P^{n}\left(\left(\omega_{1}, \ldots, \omega_{n}\right) \in \Omega^{n}:\left\|\frac{1}{n} \sum_{i=1}^{n} \xi\left(\omega_{i}\right)-\mathbb{E}_{P} \xi\right\|_{H}^{2} \geq 32 \frac{\tau^{2}}{n}\left(\sigma^{2}+\frac{L^{2}}{n}\right)\right) \leq 2 e^{-\tau}
$$
\end{theo}

\begin{lma}\label{lma:effective_dim}
Suppose \eqref{asst:evd} holds. Then, there exists a $c>0$ such that:
$$
\mathcal{N}(\lambda)=\text{Tr}\left(C_{XX}\left(C_{XX}+\lambda Id_{\mathcal{H}_X}\right)^{-1}\right) \leq c \lambda^{-p}
$$
\end{lma}

\section{Well-specifiedness of the CME problem and discussion of some corner cases} \label{sec:well_spec}

As the CME has been redefined various times, the conditions
ensuring the existence of a closed-form solution
have been subject to various modifications. 
The purpose of this
section is to briefly investigate
the well-specifiedness assumptions in the 
operator-theoretic setting \cite{song2009hilbert,klebanov2020rigorous}
and in our kernel regression setting \cite{park2020measure}.
The connections between these assumptions
are rather complex
(we also refer to Section 5 of \cite{klebanov2021linear}
and to Section 2.4 of \cite{talwai2021sobolev}).

To recapitulate, well-specifiedness in the original 
operator-theoretic setting usually involves the requirement
\[
\tag{a}
\label{eq:well-specified-op}
\mathbb{E}[g(Y) \vert X = \cdot ] \in \mathcal{H}_X 
\text{ for all } g \in \mathcal{H}_Y,
\]
while well-specifiedness in the kernel regression setting
means that a representative of the $L_2$-function class
associated with the CME function is contained 
in the hypothesis space $\mathcal{G}$, which we write for simplicity as
\[
\tag{b}
\label{eq:well-specified-reg}
\mathbb{E}[\phi_Y(Y) \vert X = \cdot ] \in \mathcal{G}. 
\]
Before we discuss some corner cases, we first point out that condition (\ref{eq:well-specified-reg}) implies condition (\ref{eq:well-specified-op}). To see this, we notice that by Corollary~\ref{theo:operep}, we have $\mathbb{E}[\phi_Y(Y) \vert X = \cdot ] = C \phi_X(\cdot)$ for some $C \in S_2(\mathcal{H}_X,\mathcal{H}_Y)$. Therefore, for any $g \in \mathcal{H}_Y$, we have \[\mathbb{E}[g(Y) \vert X = \cdot ] = \langle g, C\phi_X(\cdot) \rangle_{\mathcal{H}_Y} = \langle C^*g, \phi_X(\cdot) \rangle_{\mathcal{H}_X}. \] It is easy to see that $C^*g \in \mathcal{H}_X$ for any $g \in \mathcal{H}_Y$, hence condition (\ref{eq:well-specified-op}) is satisfied. 

\paragraph{$Y = X$:} This is an example that condition (\ref{eq:well-specified-op}) does not
imply condition (\ref{eq:well-specified-reg}).
Let $\mathcal{G}$ be the vRKHS induced by the kernel
 \[K(x,x') = k_{X}(x,x')\text{Id}_{\mathcal{H}_{Y}}, x,x' \in E.\]
The first example is the special case where we have 
$k_Y = k_X$ as well as $X = Y$.
It is easy to see that this reduces the CME to
\[
\mathbb{E}[\phi_X(X) \vert X = \cdot ] = \phi_X(\cdot ).
\]
We can also verify that condition (\ref{eq:well-specified-op})
is satisfied in this case, as we have
$\mathbb{E}[g(X) \vert X = \cdot ] = g(\cdot) \in \mathcal{H}_X $.

Furthermore, 
it is clear that the identity operator $\operatorname{Id}_{\mathcal{H}_X}$ is
the correct operator-theoretic solution to the CME problem, as it
represents the CME in terms of
$\phi_X(\cdot) = \operatorname{Id}_{\mathcal{H}_X} \varphi_X(\cdot)$.
However, if $\mathcal{H}_X$ is infinite dimensional,
it is also clear that $\operatorname{Id}_{\mathcal{H}_X}$ is not
Hilbert--Schmidt and hence
\[
\operatorname{Id}_{\mathcal{H}_X} \phi_X(\cdot) 
\notin \mathcal{G} \simeq \mathcal{H}_X \otimes \mathcal{H}_X.
\]
Hence, according to condition (\ref{eq:well-specified-op}),
we have a well-specified setting,
while according to 
condition (\ref{eq:well-specified-reg}), we clearly have a
misspecified setting.

This example demonstrates that, without additional 
requirements, the well-specifiedness condition
(\ref{eq:well-specified-op}) allows cases where the
CME is represented by a bounded operator, while
condition (\ref{eq:well-specified-reg}) restricts
the class of admissible representative operators
to the Hilbert--Schmidt class.

\paragraph{$Y \perp X$:}
In this case, it is clear that \[
\mathbb{E}[\phi_Y(Y) \vert X = \cdot ] = \int_{E_Y}\phi_Y(y)d\nu(y) = \mu_Y .
\]
Similar to the previous case, neither condition (\ref{eq:well-specified-op}) nor  (\ref{eq:well-specified-reg}) are satisfied. Moreover, requiring that the CME is contained in $ [\mathcal{G}]^{\beta}$ amounts to require that $ \mathbb{E}(g(Y) \vert X = \cdot) \in [\mathcal{H}]_X^{\beta}$. However, when $Y$ is independent of $X$, we have $\mathbb{E}(g(Y) \vert X = \cdot) = \mathbb{E}(g(Y))$ which is a constant function. Since the constant function is included in $[\mathcal{H}]_X^{\beta}$ for $\beta = 0$, essentially independence between $Y$ and $X$ is equivalent to the case where the target CME is contained in $[\mathcal{G}]^0$.

\end{document}